\def\paperID{8530}
\def\confName{CVPR}
\def\confYear{2024}

\def\paperTitle{Stationary Representations: Optimally Approximating Compatibility and Implications for Improved Model Replacements
}

\def\authorBlock{
    Niccolò Biondi 
    \quad
    Federico Pernici  \;
    \quad
    Simone Ricci 
    \quad
    Alberto Del Bimbo \\
    DINFO (Department of Information Engineering), University of Florence, Italy, \\
    MICC (Media Integration and Communication Center), \\
    {\tt name}.{\tt surname}@unifi.it
}

\newif\ifreview 
\newif\ifarxiv \newcommand{\arxiv}{\arxivtrue}
\newif\ifcamera 
\newif\ifrebuttal 

\arxiv 

\pdfoutput=1
\documentclass[10pt,twocolumn,letterpaper]{article}
\ifreview \usepackage[review]{cvpr} \fi
\ifarxiv \usepackage[pagenumbers]{cvpr} \fi
\ifrebuttal \usepackage[rebuttal]{cvpr} \fi
\ifcamera \usepackage{cvpr} \fi

\usepackage{graphicx}	
\usepackage{amsmath}	
\usepackage{amssymb}	
\usepackage{booktabs}
\usepackage{times}
\usepackage{microtype}
\usepackage{epsfig}
\usepackage[table,xcdraw,dvipsnames]{xcolor}
\usepackage{caption}
\usepackage{float}
\usepackage{placeins}
\usepackage{color, colortbl}
\usepackage{stfloats}
\usepackage{enumitem}
\usepackage{tabularx}
\usepackage{xstring}
\usepackage{multirow}
\usepackage{xspace}
\usepackage{url}
\usepackage{subcaption}
\usepackage{xcolor}
\usepackage[hang,flushmargin]{footmisc}
\usepackage{siunitx}
\usepackage{adjustbox}

\ifcamera \usepackage[accsupp]{axessibility} \fi

% ...

\usepackage{xspace}

\newcommand{\loss}{HOC\xspace}

\ifarxiv  \fi

\newcommand{\R}[1]{{%
    \textbf{%
        \ifstrequal{#1}{1}{\textcolor{red}{R#1}}{%
        \ifstrequal{#1}{2}{\textcolor{blue}{R#1}}{%
        \ifstrequal{#1}{3}{\textcolor{magenta}{R#1}}{%
        \ifstrequal{#1}{4}{\textcolor{teal}{R#1}}{%
                           \textcolor{cyan}{R#1}%
        }}}}%
    }%
}}

\newtheorem{thm}{Theorem}
\newtheorem{lem}[thm]{Lemma}

\newtheorem{defn}{Definition}

\newtheorem{innercustomgeneric}{\customgenericname}
\providecommand{\customgenericname}{}
\newcommand{\newcustomtheorem}[2]{%
  \newenvironment{#1}[1]
  {%
   \renewcommand\customgenericname{#2}%
   \renewcommand\theinnercustomgeneric{##1}%
   \innercustomgeneric
  }
  {\endinnercustomgeneric}
}

\newcustomtheorem{customthm}{Theorem}

\newcustomtheorem{customcor}{Corollary}

\usepackage{xr-hyper}

\makeatletter
\newcommand*{\addFileDependency}[1]{
  \typeout{(#1)}
  \@addtofilelist{#1}
  \IfFileExists{#1}{}{\typeout{No file #1.}}
}

\makeatother

\usepackage{listings}

\usepackage{siunitx}
\usepackage{booktabs}
\usepackage{makecell}
\usepackage{amsmath}
\usepackage{diagbox}

\definecolor{cvprblue}{rgb}{0.21,0.49,0.74}
\usepackage[pagebackref,breaklinks,colorlinks,citecolor=cvprblue]{hyperref}
\usepackage[capitalize]{cleveref}
\crefname{section}{Sec.}{Secs.}
\crefname{table}{Table}{Tables}
\crefname{figure}{Fig.}{Figs.}

\newenvironment{proof}{{\noindent\it Proof. \;}}{\nobreak\hfill$\square$}

\newcommand{\dist}{{\rm d}\xspace}

\begin{document}
\title{\paperTitle}
\author{\authorBlock}
\maketitle

\begin{abstract}
Learning compatible representations enables the interchangeable use of semantic features as models are updated over time. This is particularly relevant in search and retrieval systems where it is crucial to avoid reprocessing of the gallery images with the updated model. While recent research has shown promising empirical evidence, there is still a lack of comprehensive theoretical understanding about learning compatible representations. 
In this paper, we demonstrate that the stationary representations learned by the $d$-Simplex fixed classifier optimally approximate compatibility representation according to the two inequality constraints of its formal definition. This not only establishes a solid foundation for future works in this line of research but also presents implications that can be exploited in practical learning scenarios. An exemplary application is the now-standard practice of downloading and fine-tuning new pre-trained models. 
Specifically, we show the strengths and critical issues of stationary representations in the case in which a model undergoing sequential fine-tuning is asynchronously replaced by downloading a better-performing model pre-trained elsewhere. Such a representation enables seamless delivery of retrieval service (i.e., no reprocessing of gallery images) and offers improved performance without operational disruptions during model replacement. 
Code available at: \url{https://github.com/miccunifi/iamcl2r}.
\end{abstract}

\section{Introduction}
\label{sec:intro}
By learning powerful internal feature representations from data, Deep Neural Networks (DNNs) \cite{chopra2005learning,bengio2013representation,sharif2014cnn,YosinskiNIPS2014} have made tremendous progress in some of the most challenging search tasks such as face recognition~\cite{taigman2014deepface, sun2015deepid3, Ranjan2017, DBLP:conf/cvpr/DengGXZ19, meng2021magface}, person re-identification~\cite{Sun2018Beyond, Hermans2017In, ristani2018features}, image retrieval~\cite{radenovic2018revisiting,radenovic2018fine,chen2022deep} and this significance also extends to a variety of other data modalities \cite{brown2020language, 
 radford2021learning}. 
Although all of the works mentioned above have focused on learning feature representations from \textit{static} and, more recently, \textit{dynamic} datasets~\cite{caccia2021new, Davari_2022_CVPR, barletti2022contrastive, 
asadi2023prototype}, the now-standard practice is downloading and fine-tuning representations from models pre-trained elsewhere \cite{wolf2019huggingface,touvron2023llama}. These ``third-party'' pre-trained models often incorporate new data, utilize alternative architectures, adopt different loss functions or more in general provide novel methodologies. Whether applied individually or combined, these advancements aim to encapsulate the field's rapid progress within a single unified model \cite{raffel2023building}. This greatly facilitates the exploitation of internally learned semantic representations, particularly as models, datasets, and computational infrastructure continue to expand in size, complexity, and cost \cite{bommasani2021opportunities,sorscher2022beyond}.

\begin{figure}[t]
    \centering    \includegraphics[width=0.9\linewidth]{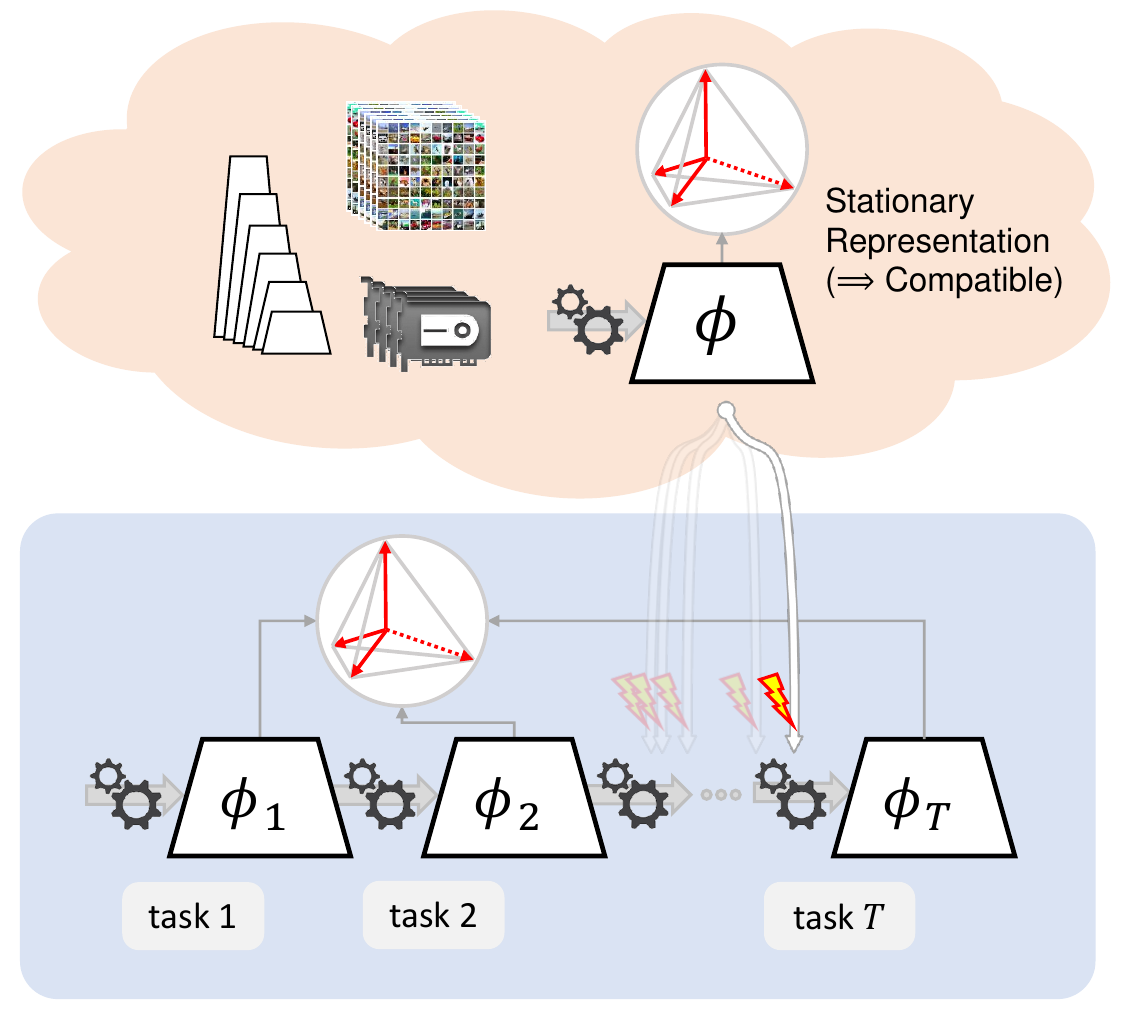}
    \caption{
    Improved Asynchronous Model Compatible Lifelong Learning Representation (IAM-CL$^2$R pronounced ``\textit{I am clear}''). 
    In the process of lifelong learning, a model is sequentially fine-tuned and asynchronously replaced with improved third-party models that are pre-trained externally. Stationary representations ensure seamless retrieval services and better performance, without the need to reprocess gallery images.
    }
    \label{fig:intro}
\end{figure}

The challenge of fully exploiting such standard practice in retrieval/search systems has to deal with the underlying problem of \textit{compatible learning} \cite{shen2020towards,meng2021learning,biondi2023cores}. That is the desire to align the representation of different models trained with different data, initialization seeds, loss functions, or alternative architectures---either individually or in combination.
In such applications, maintaining alignment is crucial to minimize the need for repeated reprocessing of gallery images for feature extraction each time a new pre-trained model becomes available \cite{raffel2023building}.
Reprocessing is not only computationally intensive but may also be unsustainable for extensive gallery sets~\cite{bommasani2021opportunities, sorscher2022beyond, strubell2019energy}
or unfeasible if the original images are no longer accessible due to privacy concerns~\cite{price2019privacy}.
This holds across various typical galleries: social networks update millions of images every month, while in robotics and automotive domains, the update rate can be as rapid as hundreds of images every second. Similarly, in textual domains, books can be structured into chapters, paragraphs, and sentences, enabling the capture of semantic relationships between these segments. 
While a similar organizational principle can be structured for the web with LLMs \cite{lewis2020retrieval,radford2021learning}, the challenge lies in the impracticality of reprocessing such extensive content with each advancement in representation models.
Although recent research has shown the effectiveness of compatible representation learning~\cite{shen2020towards, wang2020unified, biondi2023cores, meng2021learning, zhang2022towards, duggal2021compatibility, pan2023boundary, iscen2020memory, Wan_2022_CVPR, biondi2023cl2r, trauble2021backward,zhou2023bt}, there is still a lack of comprehensive theoretical understanding about compatibility. 

This paper introduces a theorem that demonstrates how the stationary representations proposed in \cite{pernici2021regular, Pernici_2019_CVPR_Workshops} optimally approximate compatibility according to the two inequality constraints of its formal definition as provided in \cite{shen2020towards}.
This not only establishes a solid foundation for future works, but also presents implications that can be exploited fine-tuning third-party models without the need of reprocessing gallery images. Specifically, we show that a continuously fine-tuned model can be asynchronously replaced by downloading a higher-performing, pre-trained model from an external source. Due to stationarity (and therefore optimal compatibility), such a replacement provides seamless retrieval services with improved performance, eliminating the need for image gallery reprocessing. We refer to this scenario as Improved Asynchronous Model Compatible Lifelong Learning Representation (IAM-CL$^2$R pronounced ``\textit{I am clear}''). Fig.~\ref{fig:intro} illustrates the relationship between sequential fine-tuning and model replacement.  
Furthermore, as will be elaborated in the related work section, our foundation draws connections with the Neural Collapse phenomenon \cite{papyan2020prevalence} and its associated theory.

Our second contribution is related to a specific challenge that arises: the tendency of the old and the new replaced models to align at their first-order statistics, an inherent property of stationary representation. Consequently, cross-entropy based prediction errors alone, when fine-tuning the representation, may not fully capture higher-order dependencies. To address this issue while preserving compatibility, we show that learning stationary representations using a convex combination of the cross-entropy loss and the infoNCE~loss~\cite{oord2018representation} is equivalent to training under one of the compatibility inequality constraints in~\cite{shen2020towards}. 
This combined loss, termed Higher-Order Compatibility (HOC), distinguishes itself from the use of cross-entropy alone by capturing higher-order dependencies and optimally approximating compatibility.

\section{Related Work}
\label{sec:related}
\textbf{Neural Collapse.} 
Neural Collapse (NC) is an empirical phenomenon that demonstrates the alignment between features and the classifier in a symmetric configuration \cite{papyan2020prevalence}.
Specifically, each class feature vector and its corresponding class prototype vector align with each other (i.e., collapse onto the same vector), forming a regular Simplex geometry in a subspace of the representation space. This particular configuration, which results in maximal separation of the collapsed vectors, is also referred to as a regular Simplex ETF (Equiangular Tight Frame). 
As training progresses and the training phase goes beyond zero classification error, the network increasingly approaches collapse. 
Notably, this also agrees with the double descent generalization regime observed within the same training phase \cite{belkin2019reconciling}. The two phenomena together indicate a form of stable steady-state for the internal representations of Deep Neural Networks.

Prior to the observation of neural collapse, other research applied the steady-state of the Simplex geometry directly from the beginning of training.
The fixed classifier with mutually orthogonal prototypes, introduced in \cite{hoffer2018fixed}, firstly demonstrates no degradation in classification performance. Building on this initial model, the regular polytope fixed classifiers---such as the $d$-Simplex, $d$-Cube, and $d$-Orthoplex---advance the concept further by observing stationary and maximally separated representations, as introduced in \cite{pernici2019fix} and further detailed in \cite{pernici2021regular}. Prior to these developments, \cite{liu2018learning} delved into the early energy-based investigations of symmetric and maximal separation in the representation space.
The distinction between the natural emergence of a regular Simplex ETF and intentionally fixing the regular Simplex geometry at the beginning of training is that prior fixing can preserve regions in the representation space for future classes, as introduced in \cite{pernici2021class} and more recently in \cite{yang2022neural} and \cite{zhou2022forward}. Our work takes advantage of this preservation for future classes, allowing third-party representation models to be trained from scratch and fine-tuned, while mitigating the interference in the representation space of the classes involved in both processes.

As neural collapse is related to the interaction between the neural network's final and penultimate layers, it offers a tool to examine training dynamics and convergence, as introduced in \cite{mixon2022neural} and \cite{fang2021exploring} under the name of Unconstrained Feature Model (UFM) and Layered Peeled Model (LPM), respectively.
In both \cite{graf2021dissecting} and \cite{zhu2021geometric}, the favorable convergence of fixing the final classifier according to the UFM is demonstrated. 
In \cite{fang2021exploring}, it is shown that training on imbalanced datasets does not necessarily result in NC. Additional observations from \cite{yang2022we} suggest that NC can emerge in both imbalanced and long-tail scenarios when the classifier is fixed to a $d$-Simplex geometry. 
Further detailed results on NC are presented in \cite{kothapalli2022neural}. Our proof is based on the assumption from the UFM and LPM that the backbone has sufficient expressiveness to allow for the independent study of each feature. Our proof is also based on the assumption of $d$-Simplex fixed classifier, whose inherent symmetry allows to reduce the extent of the analysis to a single pairwise class interaction, as it causes all the interactions to be identical.

\noindent
\textbf{Compatible Representations Learning.} 
Compatible representations broadly refer to the ability to align different learned representations, as discussed in \cite{lenc2015understanding,li2015convergent, wang2018towards,kornblith2019similarity,bansal2021revisiting}.
The distinction outlined in \cite{shen2020towards} is that the alignment of models should be achieved without wasting the information learned from new data. This capability is typically evaluated in a query and gallery setting, where query and gallery features are extracted from two different representation models. The model for the query is trained using an extended dataset that includes additional data not present in the one used for training the gallery's model.
The study in \cite{shen2020towards} further presents a method called Backward Compatible Training (BCT), which applies regularization to a new model using the classifier from the previous learning phase. This approach implicitly aligns the current improved model with the previously trained classifier, which is kept fixed. Several other methods have adopted this basic working principle: 
The fundamental aspect of this principle is that the challenge of model alignment is primarily demanded by the new model, which must learn from both the additional and the old data how to compensate for the inadequate representations of the previously learned models.
Conversely, as also recently highlighted in \cite{zhou2023bt}, methods such as \cite{Chen_2019_CVPR} or the more recent \cite{hu2022learning,ramanujan2022forward,meng2021learning} train a lightweight transformation to convert old representations into new ones for backward compatibility. However, these methods do not entirely eliminate the re-processing cost. As the number of chained mappings increases, the entire chain necessitates re-evaluation each time the representation model is updated. This makes them unsuitable for sequential learning and large gallery-sets. While its primary focus is on classification, the study in \cite{iscen2020memory} is one of the first methods employing sequential chaining transformations for aligning representations within a common reference space.
The works in \cite{biondi2023cl2r} and \cite{Wan_2022_CVPR} bypass the use of chaining transformations, focusing instead on aligning representations for compatibility purposes in lifelong learning scenarios. Both approaches leverage auxiliary losses to ensure similarity among previously learned representations. 
Additionally, \cite{biondi2023cl2r} achieves alignment with an absolute reference through the use of fixed classifiers, in line with the neural collapse phenomenon.

The work in \cite{zhou2023bt} argues that there is an inherent trade-off in the definition of compatibility introduced in \cite{shen2020towards}, which inspires them to ``hold'' incompatible information of the new model on additional orthogonal dimensions to avoid this conflict. Their argument seems to be in line with the recent work \cite{biondi2023cores} and \cite{biondi2023cl2r}
based on stationarity in which (nearly) orthogonal dimensions are pre-allocated from the beginning using a regular $d$-Simplex fixed classifier.
In this paper, we establish a formal relationship among compatibility, neural collapse, and stationarity, showing that stationarity provides an optimal approximation to the compatibility definition formulated in \cite{shen2020towards}.

\section{Theoretical Results}
\label{sec:method}

\subsection{Stationarity and Compatibility}

\begin{figure*}[t]
    \centering    
    \hspace{-25pt}
    \begin{subfigure}{0.57\linewidth}
    \centering
    \adjincludegraphics[height=4.2cm,trim={0 0 {.65\width} 0},clip]   {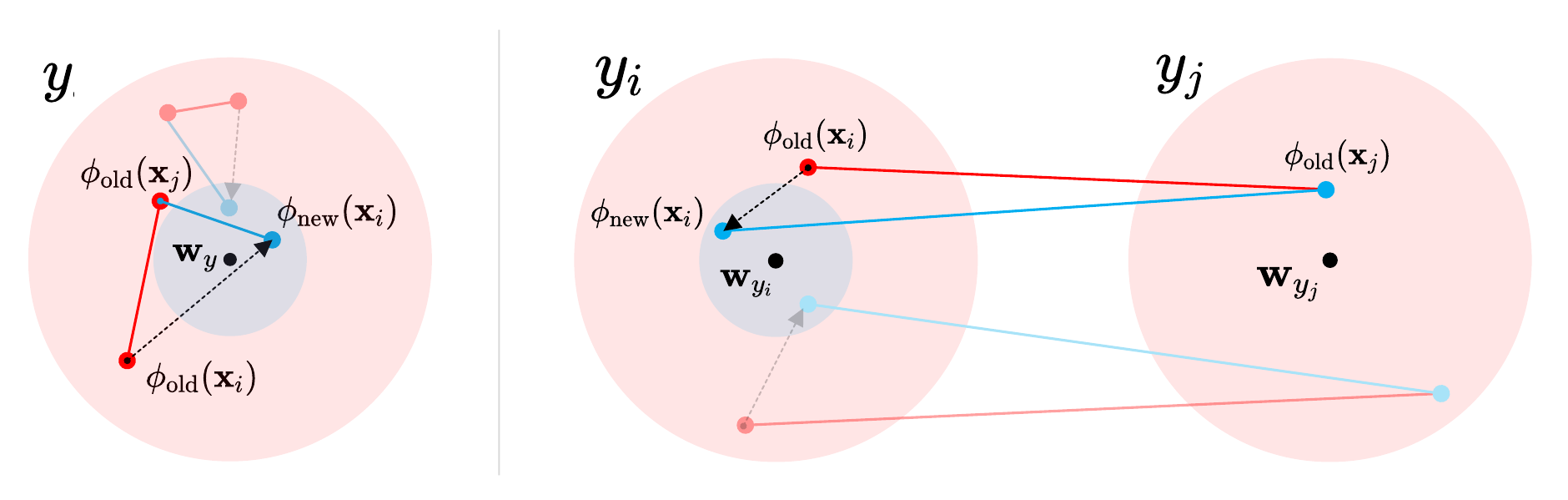}
    \caption{Same class}\label{fig:feat_space_change_pre}
    \end{subfigure}
    \hspace{-98pt}
    \begin{subfigure}{0.57\linewidth}
    \centering        
    \adjincludegraphics[height=4.2cm,trim={{.35\width} 0 0 0},clip]        {figs/proof-fig.pdf}
    \caption{Different class}
    \label{fig:feat_space_change_post}
    \end{subfigure}
    \caption{ 
    Key concepts and relationships underlying Theorem~\ref{theo:compatibility}. Distances in feature space of two distinct samples within their hyperballs before and after model update, with the update process represented by a dotted arrow.
    (\textit{a}): Distances between samples $\mathbf{x}_i$ and $\mathbf{x}_j$ of the same class $y$ before (red) and after (cyan) model update.
    (\textit{b}): Distances between samples $\mathbf{x}_i$ of class $y_i$ and $\mathbf{x}_j$ of class $y_j$, before (red) and after (cyan) model update. Compatibility is verified by computing the expected lengths of the segments and verifying if they satisfy the inequalities of the compatibility definition. 
    A transparently colored instance shows counter-intuitive distance behavior. Expectation reveals the underlying pattern of approximation. 
    \vspace{-0.3cm}
    }
    \label{fig:theo-diff-classes-dist}
\end{figure*}
\paragraph{Preliminaries.}
Let $\mathcal{G}=\{ \mathbf{x}_i \}_{i=1}^{N_g}$ be a gallery-set composed of a set of $N_g$ images $\mathbf{x}_i \in \mathbb{R}^D$ with class labels from $\mathcal{Y} = \{ y_i \}_{i=1}^{L}$
and $\Phi^\mathcal{G} = \{ \phi(\mathbf{x}_i) \in \mathbb{R}^d \, | \, \forall \mathbf{x}_i$ $\in \mathcal{G}\}$ be the set of feature vectors of the gallery-set $\mathcal{G}$ obtained with representation model $\phi$. Let $\mathcal{Q}=\{ \mathbf{x}_i \}_{i=1}^{N_q}$ be a query-set  composed of $N_q$ images $\mathbf{x}_i \in \mathbb{R}^D$ and $\Phi^\mathcal{Q} = \{ \phi(\mathbf{x}_i) \in \mathbb{R}^d \, | \, \forall \mathbf{x}_i$ $\in \mathcal{Q}\}$ be the set of feature vectors of the query-set $\mathcal{Q}$ obtained with $\phi$. 
Visual search is performed using a distance function 
$\dist(\cdot, \cdot)$ to identify the closest gallery features to the query features.

Let $ \mathcal{T}_1, \mathcal{T}_2, \dots, \mathcal{T}_T $ be a sequence of $T$ tasks, where each task $\mathcal{T}$ is composed of labeled images $\mathbf{x}_i$ of class $y_i \in \mathcal{K}$ with $\mathcal{K}$ the set of classes in $\mathcal{T}$.
At task $t$, the model $\phi_t$ is fine-tuned starting from the previous representation model $\phi_{t-1}$. 
Compatibility between the current model $\phi_{t}$ and a previous model $\phi_{k}$, with $k<t$, is achieved when the feature vector of any query image obtained with $\phi_{t}$, the set $\Phi_t^\mathcal{Q}$, can be compared with feature vectors in $\Phi_k^\mathcal{G}$ without reprocessing the gallery-set. 
The following provides a formal definition of compatibility \cite{shen2020towards}:
\begin{defn}[Compatibility] \label{def:compatibility-shen}
Given two representation models $\phi_t$ and $\phi_k$, with $\phi_t$ learned after $\phi_k$, $\phi_t$ and $\phi_k$ are compatible according to the distance function $d(\cdot, \cdot)$ if it holds:
\begin{subequations}\label{eq:compatible_set_dist}
\begin{align}
 \dist \big(\phi_{k}(\mathbf{x}_i), \phi_{t}(\mathbf{x}_j) \big) &\leq \dist \big(\phi_{k}(\mathbf{x}_i), \phi_{k}(\mathbf{x}_j) \big)  \label{eq:first} \\
  \forall \, &(i,j) \in \big\{(i,j) \, | \, y_i = y_j \big\}  \notag \\   &\text{and}  \notag \\
  \dist \big(\phi_{k}(\mathbf{x}_i), \phi_{t}(\mathbf{x}_j) \big) &\geq \dist \big(\phi_{k}(\mathbf{x}_i), \phi_{k}(\mathbf{x}_j) \big) \label{eq:second}   \\
  \forall \, &(i,j) \in \big\{(i,j) \, | \, y_i \neq y_j \big\}  \notag \\ 
 \text{with  } k < t, \,\,\ t = ( 2,3, \dots,& T ) , \,\, k = ( 1,2, \dots, T-1 ).  \notag
\end{align}
\end{subequations}
\end{defn}

\paragraph{Main Result.}
In this paragraph, we state and prove that learning stationary feature representations according to a $d$-Simplex fixed classifier necessarily implies optimal approximation of the compatibility as defined in Eqs.~\ref{eq:first} and~\ref{eq:second}. 
The formulation involves examining the expected distance between feature points before and after a learning update in a high-dimensional space, where the feature points are assumed to be distributed in hyperballs (i.e., high dimensional ball) centered at the prototypes of the $d$-Simplex fixed classifier. 
This abstraction allows for mathematical manipulation and analysis of the cluster as a single entity rather than individual points. 

\begin{customthm}{1}[Stationarity $\implies$ Compatibility] 
\label{theo:compatibility}
Let $\mathbf{W}=[ \mathbf{w}_1, \mathbf{w}_2, \ldots, \mathbf{w}_K ]$ be the $d \times K$ matrix of a $d$-Simplex fixed classifier with $K$ pre-allocated classes.
Given two tasks, \( \mathcal{T}_k \) and \( \mathcal{T}_t \). The task \( \mathcal{T}_t \) is derived from \( \mathcal{T}_k \) by incorporating an additional training set \( \Delta\mathcal{T} \), such that \( \mathcal{T}_t = \mathcal{T}_k \cup \Delta\mathcal{T} \). The combined task, \( \mathcal{T}_t \), comprises a set of classes each denoted by $y$, where \( {y} \in \{ 1,2,\dots,K_t \} \) and \( K_t < K \).
Under the assumption that learning the new task \( \mathcal{T}_t \) causes the hyperball \( \mathcal{B}_k(\mathbf{w}_y) \) with radius \( r_k^y \) to shrink into a smaller hyperball \( \mathcal{B}_t(\mathbf{w}_y) \), i.e., \( r_{t}^y \leq r_k^y \) for all \( y \) in the set \( \{ 1,2, \dots, K_k \} \), then it necessarily follows that \( \phi_{t} \) and \( \phi_{k} \) optimally approximate the compatibility inequality constraints as defined in Def.~\ref{def:compatibility-shen} in expectation.
\end{customthm}
The proof is available in the Appendix.

\paragraph{Discussion.}
The Theorem relies on two main assumptions: the use of a $d$-Simplex fixed classifier~\cite{pernici2021regular} and the model's sufficient expressiveness, as described in the UFM abstraction~\cite{mixon2022neural, fang2021exploring}. 
The latter assumption enables us to consider features independently\footnote{Essentially, the Neural Collapse phenomenon, which is observed across various networks and datasets, \textit{also appears} in a two-layer neural network when assuming input feature independence (i.e., a UFM). This equivalence supports the assumption that: 1) real network backbones are typically expressive enough to learn features as independent entities, and 2) UFM can be used as a tool to study neural networks properties.}.
While the former allows focusing on a single pairwise class interaction, since interactions with all other classes are symmetrically similar and cannot change. Fig.~\ref{fig:theo-diff-classes-dist} illustrates the key concepts and relationships presented in Theorem~\ref{theo:compatibility}. 

Without loss of generality, the Theorem considers two distinct hyperballs of different radius $\mathcal{B}_{\rm new}(\mathbf{w}_y)$ and $\mathcal{B}_{\rm old}(\mathbf{w}_y)$ representing the semantic clusters of a generic class $y$, respectively before and after a generic learning update.
The assumption that features are distributed in hyperballs stems from the margin-based softmax loss\footnote{The margin enforces the confinement of features within a hyperball or a hyperdisc (the local approximation of a hypercap) around class prototypes. A disc in high-dimensional space can be considered a hyperball when referring to its filled volume.} introduced in \cite{Liu2016large_rebuttal}. This interpretation has since been utilized in various studies, such as SphereFace~\cite{Liu2017sphere_rebuttal} and ArcFace~\cite{DBLP:conf/cvpr/DengGXZ19}. Besides the margin formulation, empirical evidence, such as Neural Collapse \cite{papyan2020prevalence}, shows that class features not only cluster around their associated prototypes but also, with sufficient training epochs, collapse into them, resulting in hyperballs tightening around the prototypes.
Due to the stationarity property induced by the $d$-Simplex classifier  \mbox{$\mathcal{B}_{\rm new}(\mathbf{w}_y)$} and $\mathcal{B}_{\rm old}(\mathbf{w}_y)$ hyperballs have the same center in the representation space on the classifier prototype $\mathbf{w}_y$. 
After the learning step, $\mathcal{B}_{\rm new}(\mathbf{w}_y)$ has a shorter radius (i.e., adding new information improves the discrimination capability of the model \cite{hestness2017deep, prato2021scaling, nakkiran2021deep, caballero2023broken}).

In particular, Fig.~\ref{fig:feat_space_change_pre} shows the case in which feature vectors are from samples of the same class. 
As defined in Eq.~\ref{eq:first} compatibility requires that, after updating, the distance between $\phi_{\rm new}(\mathbf{x}_i)$ (in the cyan hyperball) and $\phi_{\rm old}(\mathbf{x}_j)$ (in red hyperball) is less than or equal to distance between $\phi_{\rm old}(\mathbf{x}_i)$ and $\phi_{\rm old}(\mathbf{x}_j)$.
The figure displays two configurations: one where the condition is met and another where it is not met (shown in transparent colors). 

Fig.~\ref{fig:feat_space_change_post} shows the case in which the feature vectors are from samples of different classes. 
As defined in Eq.~\ref{eq:second} compatibility requires that, after updating, the distance between $\phi_{\rm new}(\mathbf{x}_i)$ of class $y_i$ (in the cyan hyperball centered in $\mathbf{w}_{y_i}$) and $\phi_{\rm old} (\mathbf{x}_j)$ of class $y_j$ (in the red hyperball centered in $\mathbf{w}_{y_j}$) is greater than or equal to than the distance between $\phi_{\rm old}(\mathbf{x}_i)$ and $\phi_{\rm old}(\mathbf{x}_j)$. 
The Theorem establishes that, on average, this condition cannot be optimally satisfied and that stationarity is the best approximation achievable under the given constraints. A detailed justification for this is provided in the proof of Theorem~\ref{theo:compatibility}, with a clearer and more focused exposition presented as a  Corollary~\ref{corollary}.

Informally, the proof of the Theorem starts with the premise that, upon retraining a model, the probability of finding a class feature near the corresponding class prototype from the old model---an indicator of compatibility between the two models---is nearly zero. Subsequently, the proof establishes that the optimal approximation for a compatible representation is obtained when the average distance between the same hyperball in two distinct learned models is minimized. This minimization occurs when the two corresponding hyperballs are centered at the same class prototype and when adding more classes does not alter this distance, i.e., the stationarity condition.

Our formulation calculates the average distance between hyperballs based on the Ball Line Picking problem, which determines the expected length of a line segment that connects two random points inside a hyperball~\cite{burgstaller2009average,solomon1978geometric,kendallgeometrical,sors2004integral,kirby1982average,vaughan1984approximate}. Differently from that problem, our theorem considers a line segment connecting two random points in two distinct hyperballs, each with a different radius. Specifically, we analyze the cases as shown in Fig.~\ref{fig:theo-diff-classes-dist}. These hyperballs represent the ``class-state'' before and after the learning step during each model update. Closed-form solutions are not available for this problem, except in a specific two-dimensional case \cite{fairthorne1964distances}.

\subsection{Stationarity and Higher-Order Alignment}
A specific challenge arises when fine-tuning stationary learned representation models, for example in the \mbox{IAM-CL$^2$R} setting of Fig.~\ref{fig:intro}.
In this case the old and the new models align at the first-order statistics, an inherent property of stationarity \cite{pernici2021regular}. 
The consequence is that cross-entropy based prediction errors may not fully capture higher-order dependencies in representation space. We conjecture that simple cross-entropy mostly focuses on prediction errors related to the forgetting of the internal representation which may not promote compatibility when the representation model is largely aligned.
To address this problem, we show that adding the infoNCE loss function~\cite{oord2018representation,tian2019contrastive} is equivalent to training with the cross-entropy loss under one of the compatibility constraints while capturing higher-order dependencies.

The loss for training at task $t$ the stationary representation model $\phi_t$ assumes the form \cite{pernici2021regular}:
\makeatletter
\newcommand{\vst}{\bBigg@{2}}
\newcommand{\vast}{\bBigg@{3.5}}
\newcommand{\Vast}{\bBigg@{5}}
\makeatother
\begin{equation} \label{eq:loss_ce_simplex} 
\resizebox{0.99\hsize}{!}{$%
\begin{aligned}
    &\mathcal{L}_{\textsc{sce}} (\phi_t)= \\ 
    &= - \sum\limits_{B} \log   \vast( 
    \dfrac {\text{exp} \Big( \mathbf{W}_{y_i}^{\top}{{\phi_t(\mathbf{x}_i)}} \Big) } {\sum\limits_{\scriptscriptstyle j =1}^{K_t} \text{exp} \Big( \mathbf{W}_{j}^{\top}{{\phi_t(\mathbf{x}_i)}} \Big)  + \sum\limits_{\scriptscriptstyle j=K_t+1}^{K} \text{exp} \Big( \mathbf{W}_{j}^{\top}{{\phi_t(\mathbf{x}_i)}} \Big) } \vast)
\end{aligned}
$}
\end{equation}
where $\mathbf{W}^{\top}_{j} \in \mathbb{R}^d$ denotes the $j$-th column of the $d$-Simplex classifier matrix $\mathbf{W} \in \mathbb{R}^{d \times K}$, being $K$ the number of pre-allocated classes, $K_t = |\bigcup_{i=1}^t \mathcal{K}_i|$ the number of classes learned until time $t$ with $K_t < K$,
and $B$ is a mini-batch of samples of $\mathcal{T}_{t}$. 
The first term in the denominator accounts for the classes learned until $t$. 
The second term accounts for future classes, preserving dedicated regions in the representation space. This ensures that adding new classes minimally impacts the representation of previously learned classes \cite{pernici2021class,biondi2023cores,yang2023neural,biondi2023cl2r}.

We train the representation model $\phi_t$ with the following convex combination, namely: 
\begin{align} \label{eq:total_loss}
     \mathcal{L}_\textsc{hoc} (\phi_t) =     \lambda \mathcal{L}_{\textsc{sce}}(\phi_t)  + (1 - \lambda) \, \mathcal{L}_\textsc{nce} (\phi_t, \phi_{t-1}), & \\  \text{with} \quad \lambda \in [0,1] \nonumber
\end{align}
where: $\mathcal{L}_\textsc{sce} (\phi_t)$ is the cross-entropy loss of Eq.~\ref{eq:loss_ce_simplex}, and 
\begin{align}
    \nonumber \\
    &\mathcal{L}_\textsc{nce} (\phi_t, \phi_{t-1}) = - \sum\limits_{B} 
        \log \!  \vast( 
        \dfrac {\Delta \big( \phi_{t-1}(\mathbf{x}_i), \phi_t(\mathbf{x}_i) \big)  } { \sum\limits_{j \neq i}  \Delta \big( \phi_{t-1}(\mathbf{x}_i), \phi_t(\mathbf{x}_j) \big)   } \vast)  \label{eq:CONTRAST}
\end{align}
with 
\begin{equation}
\label{eq:CONTRAST2}
    \Delta \big( \phi_{t-1}(\mathbf{x}_i), \phi_t(\mathbf{x}_j) \big) = \text{exp} \vst( \tau  \cdot \frac{\phi_{t-1}(\mathbf{x}_i) \phi_{t}(\mathbf{x}_j)}{ ||\phi_{t-1}(\mathbf{x}_i) || \, || \phi_{t}(\mathbf{x}_j))|| }  \vst)
\end{equation}
is the contrastive loss~\cite{oord2018representation,tian2019contrastive} based on $\tau$-scaled cosine similarity between $\phi_{t-1}(\mathbf{x}_i)$ and $\phi_{t}(\mathbf{x}_j)$.
We show that training the representation model with the $\mathcal{L}_\text{\loss}$~of Eq.~\ref{eq:total_loss}
is both: (1) able to capture higher-order dependencies between old and new model representations and (2) equivalent to learning under the compatibility constraints in Def.~\ref{eq:first}.
We refer to this loss as the Higher-Order Compatibility loss ($\mathcal{L}_\textsc{hoc}$).

Through Theorem~\ref{theo:compatibility} presented in the previous section, we establish that the constraint of Eq.~\ref{eq:first} cannot be exploited in combination with the constraint of Eq.~\ref{eq:second}.  Based on this result we show that, under no specific conditions, the constrained optimization problem using solely the inequality constraint of Eq.~\ref{eq:first}:
\begin{equation}\label{eq:constrained_opt_probl}
\begin{aligned}
\underset{\phi_t}{\text{argmin}} \quad
&  \mathcal{L}_\textsc{sce}(\phi_t) 
\\
     \text{s.t.} \quad & \dist \big(\phi_k(\mathbf{x}_i), \phi_t(\mathbf{x}_j) \big)   - \dist \big(\phi_k(\mathbf{x}_i),  \phi_k(\mathbf{x}_j) \big) \leq 0 \; \\ 
    \forall \, y_i & = y_j
 \\
\end{aligned}
\end{equation}
can be transformed into a tractable form. Rooted in the work of \cite{jiang2021churn}, this transformation not only provides an approach to solve the tractability issue but, within the context of compatibility, it also allows preserving the optimality as outlined in the proof of  Theorem~\ref{theo:compatibility}. 
As shown in~\cite{jiang2021churn}, the model for a constrained problem like Eq.~\ref{eq:constrained_opt_probl} can be equivalently learned with a convex combination of the cross-entropy loss and the Kullback-Leibler divergence function. 

On the other hand, as discussed in~\cite{tian2020contrastive}, the contrastive loss $\mathcal{L}_\textsc{nce} (\phi_t, \phi_{t-1})$ can be approximated as the Kullback-Leibler divergence between the product of the marginals of the joint distribution of $\phi_t$ and $\phi_{t-1}$. Moreover, $\mathcal{L}_\textsc{nce} (\phi_t, \phi_{t-1})$ also approximates the mutual information between \(\phi_t\) and \(\phi_{t-1}\), thereby enabling to capture higher-order dependencies between consecutive updates of the model.
As a consequence, training with the loss in Eq.~\ref{eq:total_loss} is equivalent to the optimal classifier for the constrained optimization problem stated in Eq.~\ref{eq:constrained_opt_probl} and at the same time, thanks to the term $\mathcal{L}_\textsc{nce} (\phi_t, \phi_{t-1})$, takes into account higher-order variations between $\phi_{t-1}(\mathbf{x}_i)$ and $\phi_{t}(\mathbf{x}_j)$.
In the following, we call training the representation model using  $d$-Simplex with  $\mathcal{L}_\textsc{hoc}$~as $d$-Simplex-HOC.

\begin{figure}[t]
    \centering
    \includegraphics[width=0.7\linewidth]{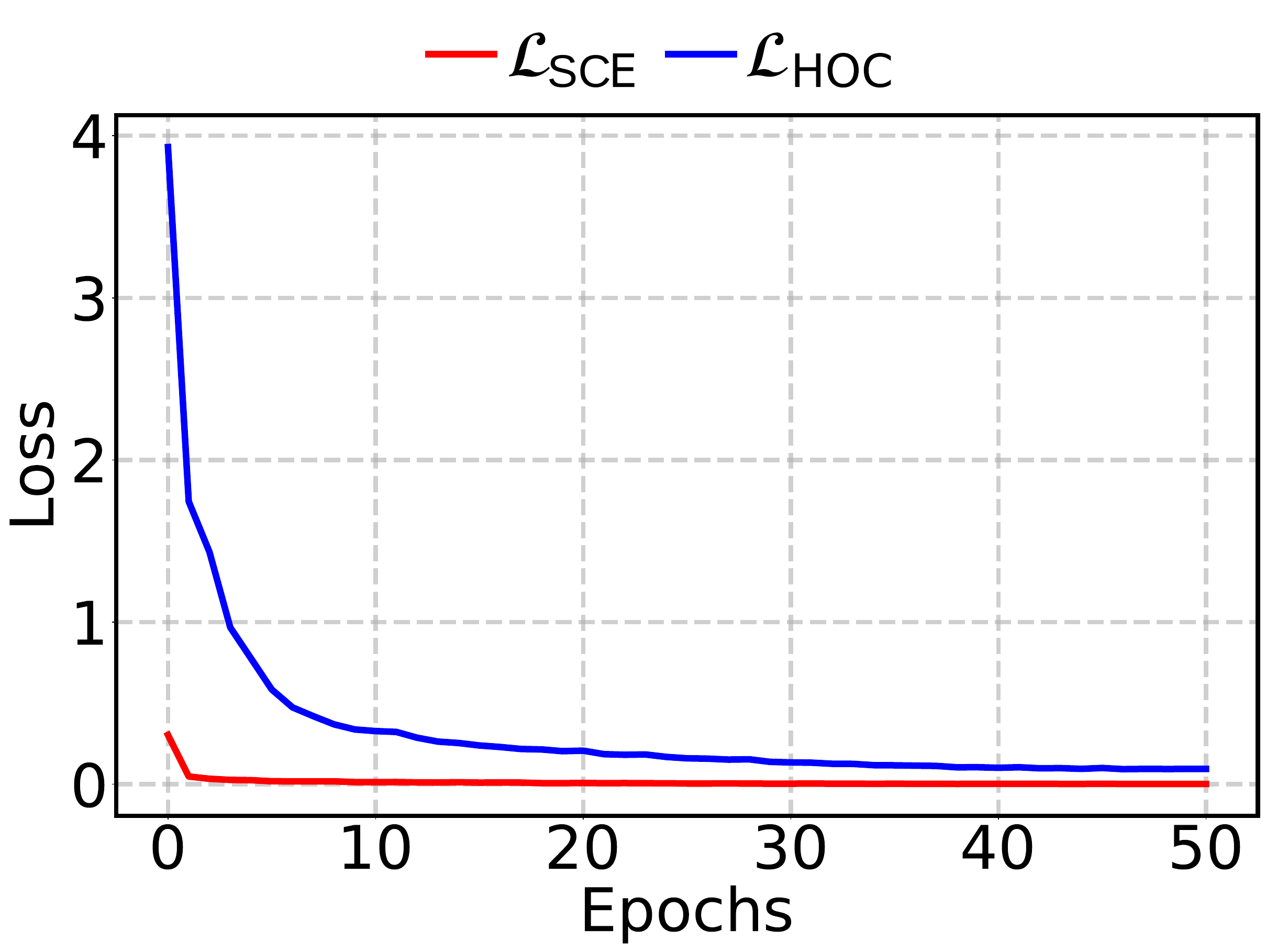}
    \caption{Training loss of a $d$-Simplex fixed classifier during a model update. Values are the cross-entropy loss of Eq.~\ref{eq:loss_ce_simplex} (red line) and the loss of Eq.~\ref{eq:total_loss} (blue line). Models are trained on MNIST. 
    }
    \label{fig:toyproblem}
\end{figure}

In Fig.~\ref{fig:toyproblem}, we illustrate the effects of $\mathcal{L}_\textsc{hoc}$ compared to the cross-entropy loss. We use a toy example with the LeNet++ CNN architecture \cite{wen2016discriminative} with the $d$-Simplex fixed classifier. The model is initially trained on the first five MNIST classes and then fine-tuned on all ten classes. The cross-entropy training error (red curve) converges rapidly to low values. In contrast, the convergence with the $\mathcal{L}_\textsc{hoc}$ loss (blue curve) is more gradual, which allows for the capture of richer information during back-propagation.

\section{Experimental Verification}
Referring to the IAM-CL$^2$R learning scenario presented in Fig.~\ref{fig:intro}, this section provides empirical evidence to verify the practical implications of the theoretical results discussed earlier.
\subsection{Datasets and Settings}
\paragraph{Pre-trained Models.}
We pre-train our models in a supervised manner using the ImageNet32~\cite{chrabaszcz2017downsampled}.
Three distinct models are pre-trained on ImageNet32 with $100$, $300$, and $600$ classes. The model trained with $100$ classes is used to initialize the model before fine-tuning on the sequence of tasks. 
The other two models are used to simulate the practice of downloading and fine-tuning pre-trained models and serve as third-party models that will replace the current one undergoing fine-tuning.
\begin{figure*}[t]
    \centering
    \begin{subfigure}{\linewidth}
        \centering        \includegraphics[width=0.82\linewidth]{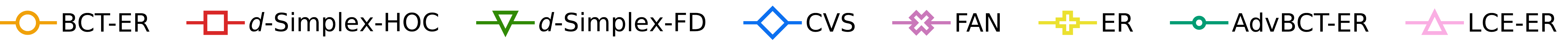}
    \end{subfigure}
    \begin{subfigure}{0.45\linewidth}
        \centering
        \includegraphics[width=0.99\linewidth]{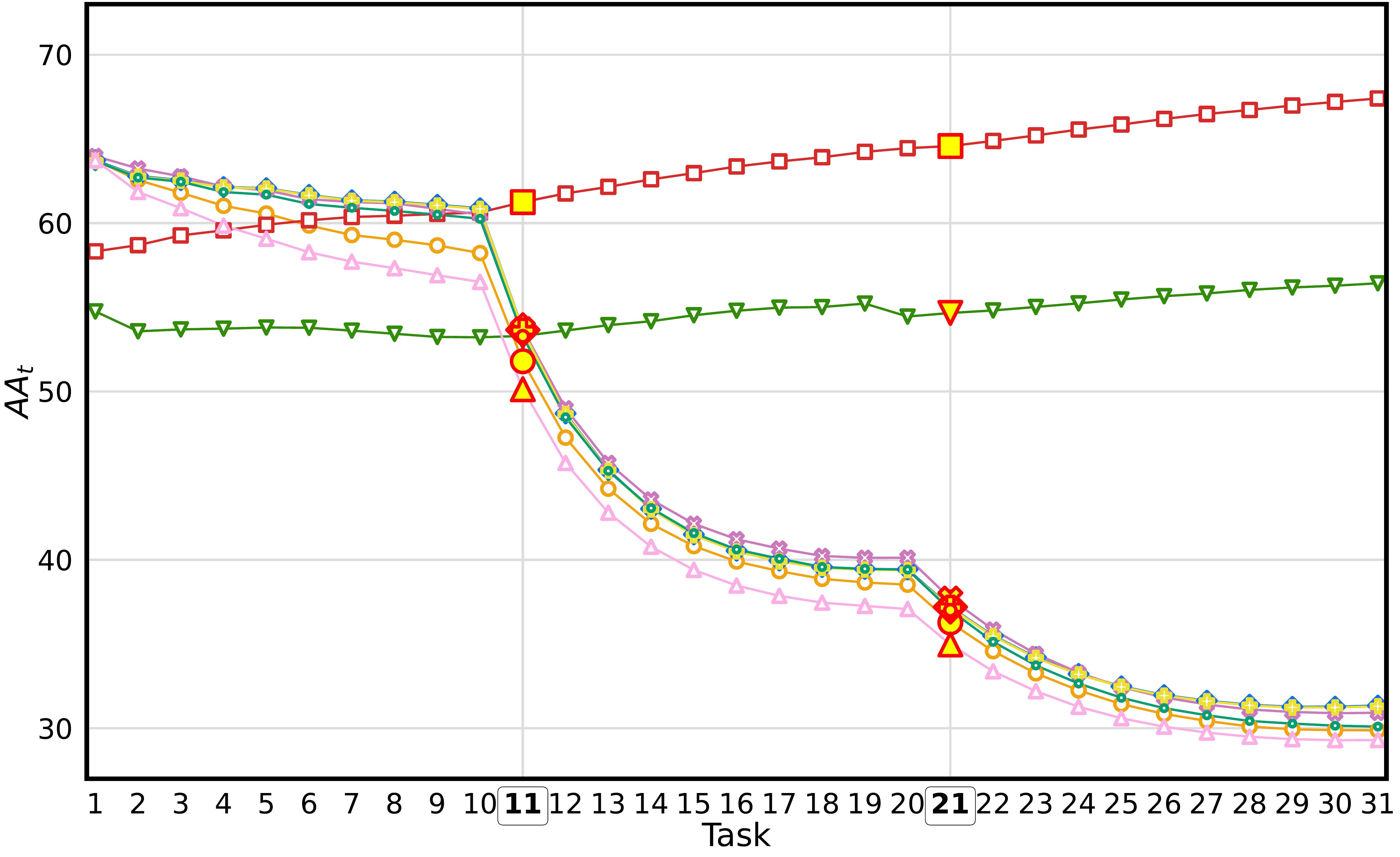} \caption{Two replacements}\label{fig:cifar30-iamcl2r}
    \end{subfigure}
    \hspace{25pt}
    \begin{subfigure}{0.45\linewidth}
        \centering        \includegraphics[width=0.99\linewidth]{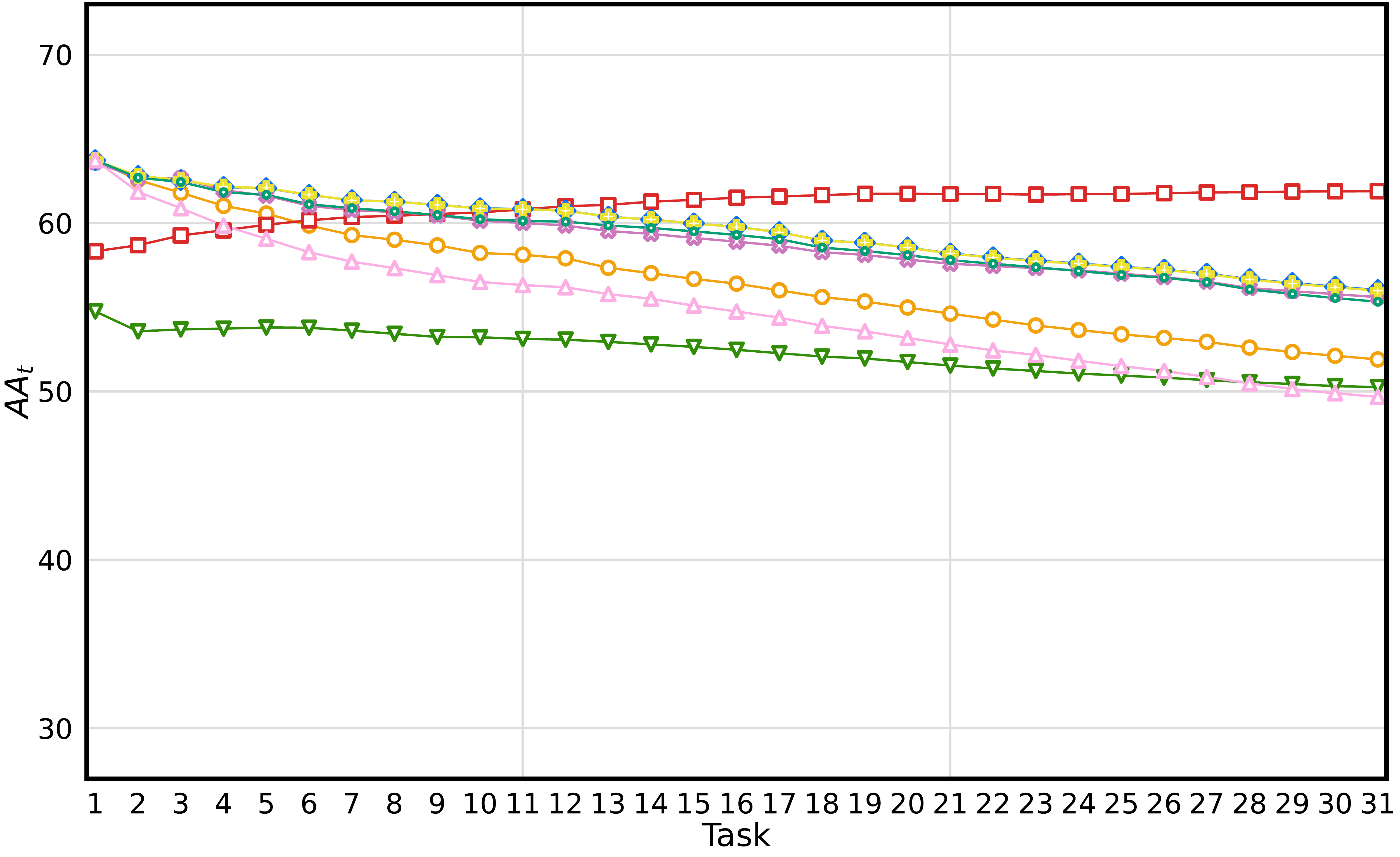} \caption{No replacements}\label{fig:cifar30-cl2r}
    \end{subfigure}
    \caption{Average multi-model Accuracy ($AA_{t}$) evaluated across 31 tasks using CIFAR100R/10, showing: \textit{(a)} model replacements at tasks 11 and 21 (indicated by yellow markers); \textit{(b)} no model replacement.
    }
    \label{fig:iamcl2r}
\end{figure*}
\begin{figure*}[t]
    \centering
    \hspace{-7pt}
    \begin{subfigure}{0.245\linewidth}
        \centering
        \includegraphics[width=0.95\linewidth]{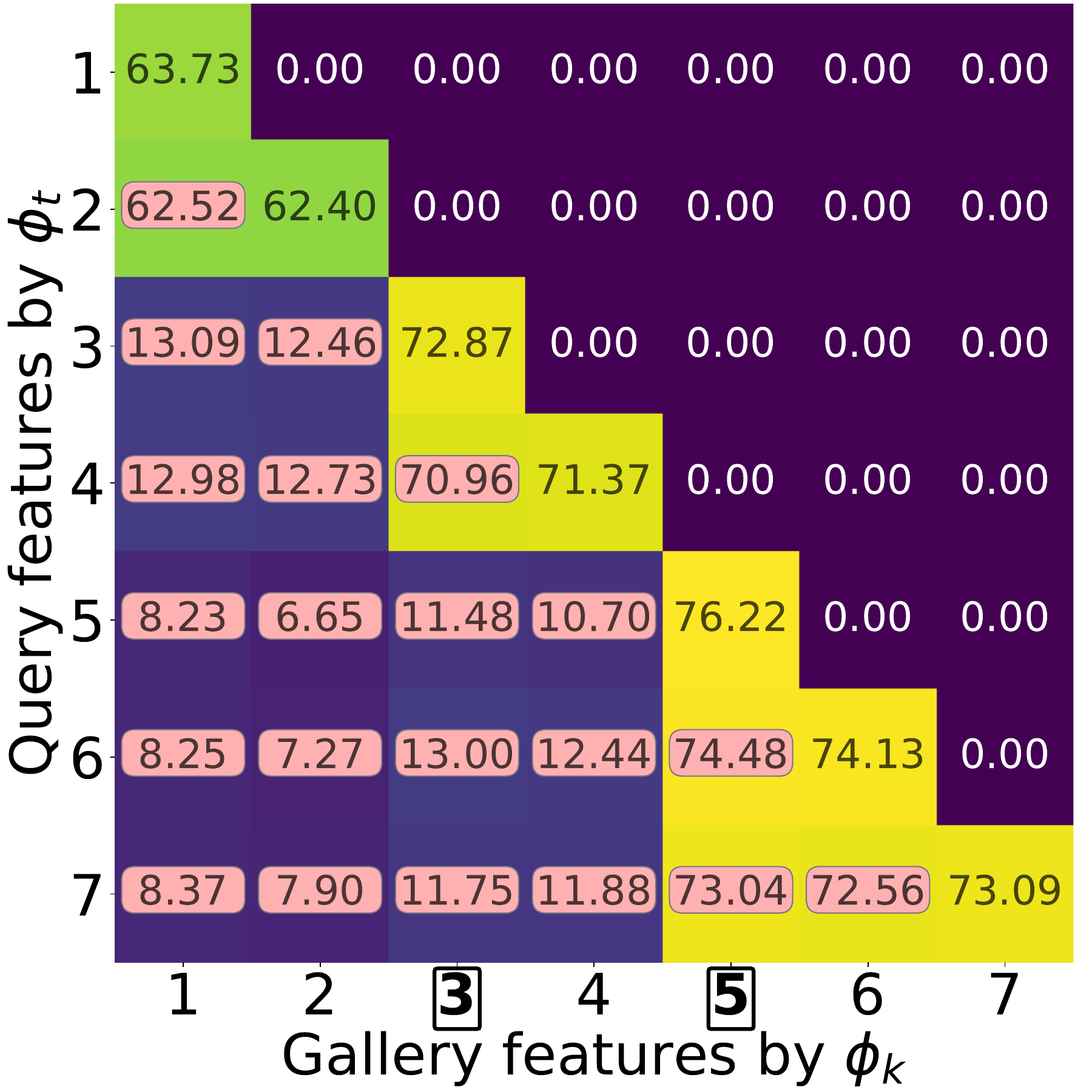}
        \caption{CVS~\cite{Wan_2022_CVPR}}
    \label{fig:cvs-10cifar-iamcl2r}
    \end{subfigure}
    \hspace{-10pt}
    \begin{subfigure}{0.245\linewidth}
        \centering
        \includegraphics[width=0.95\linewidth]{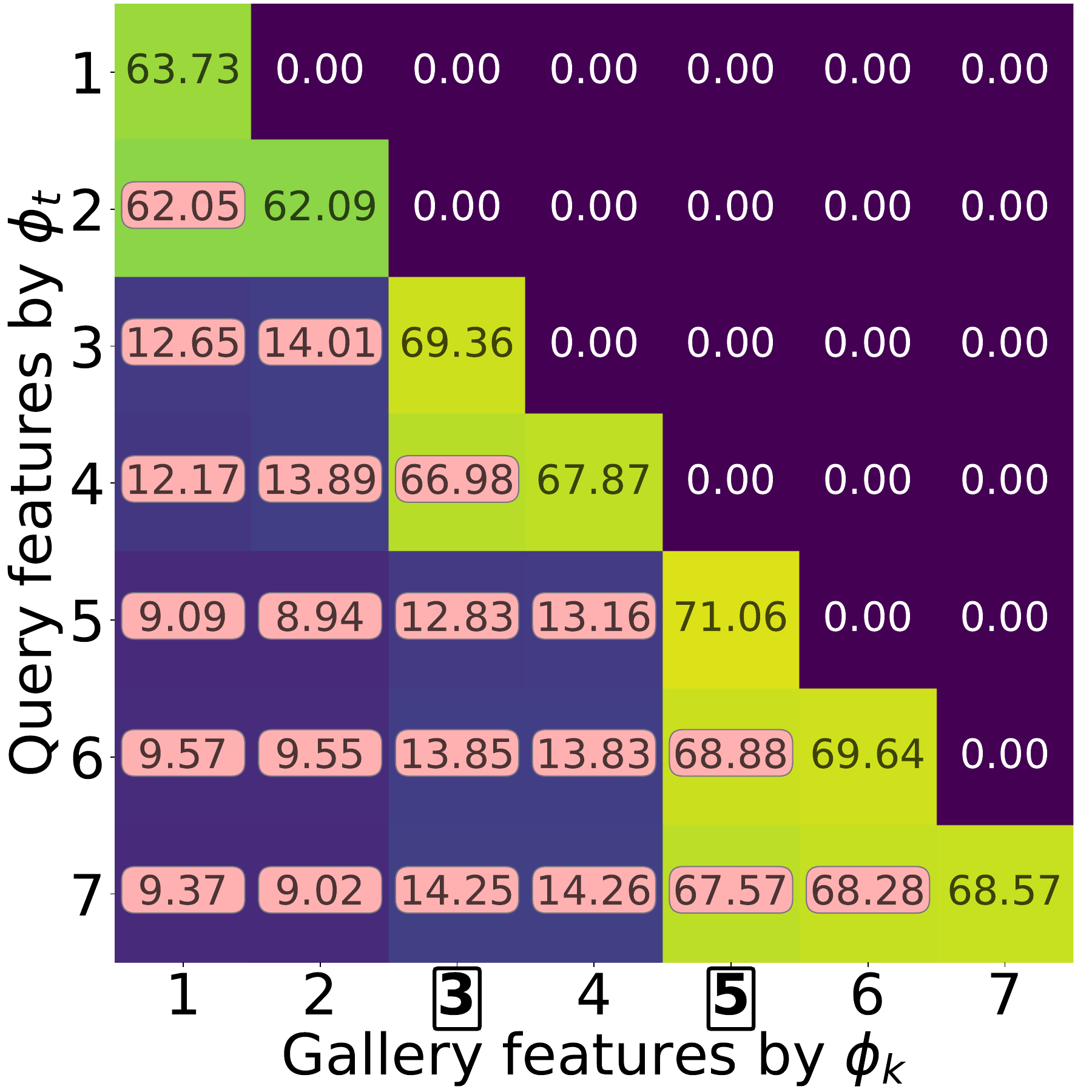}
        \caption{BCT-ER~\cite{shen2020towards}}
    \label{fig:bct-10cifar-iamcl2r}
    \end{subfigure}
    \hspace{-10pt}
    \begin{subfigure}{0.245\linewidth}
        \centering
        \includegraphics[width=0.95\linewidth]{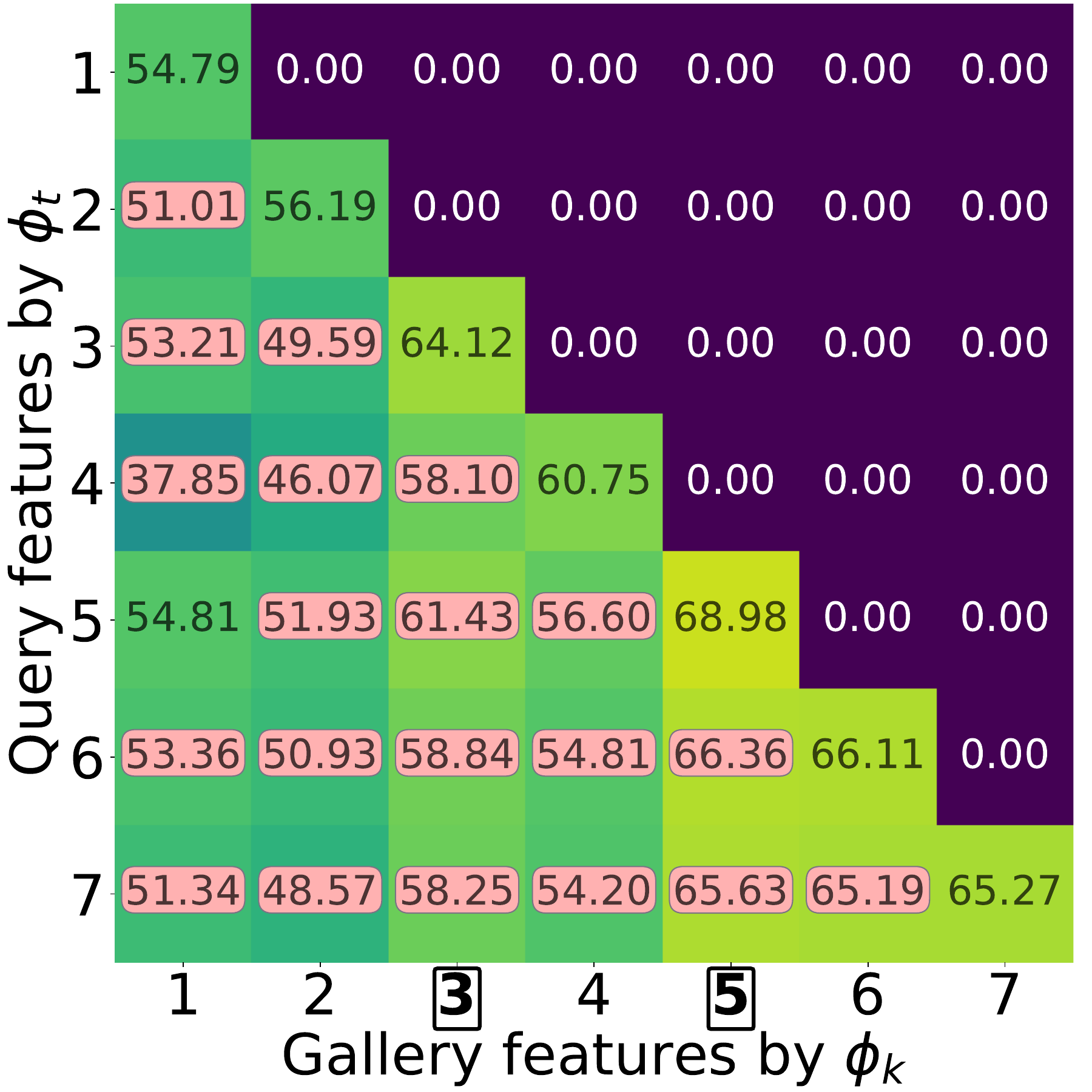}
        \caption{$d$-Simplex-FD~\cite{biondi2023cl2r}}
    \label{fig:simplex-fd-10cifar-iamcl2r}
    \end{subfigure}
    \hspace{-10pt}
    \begin{subfigure}{0.245\linewidth}
        \centering
        \includegraphics[width=0.95\linewidth]{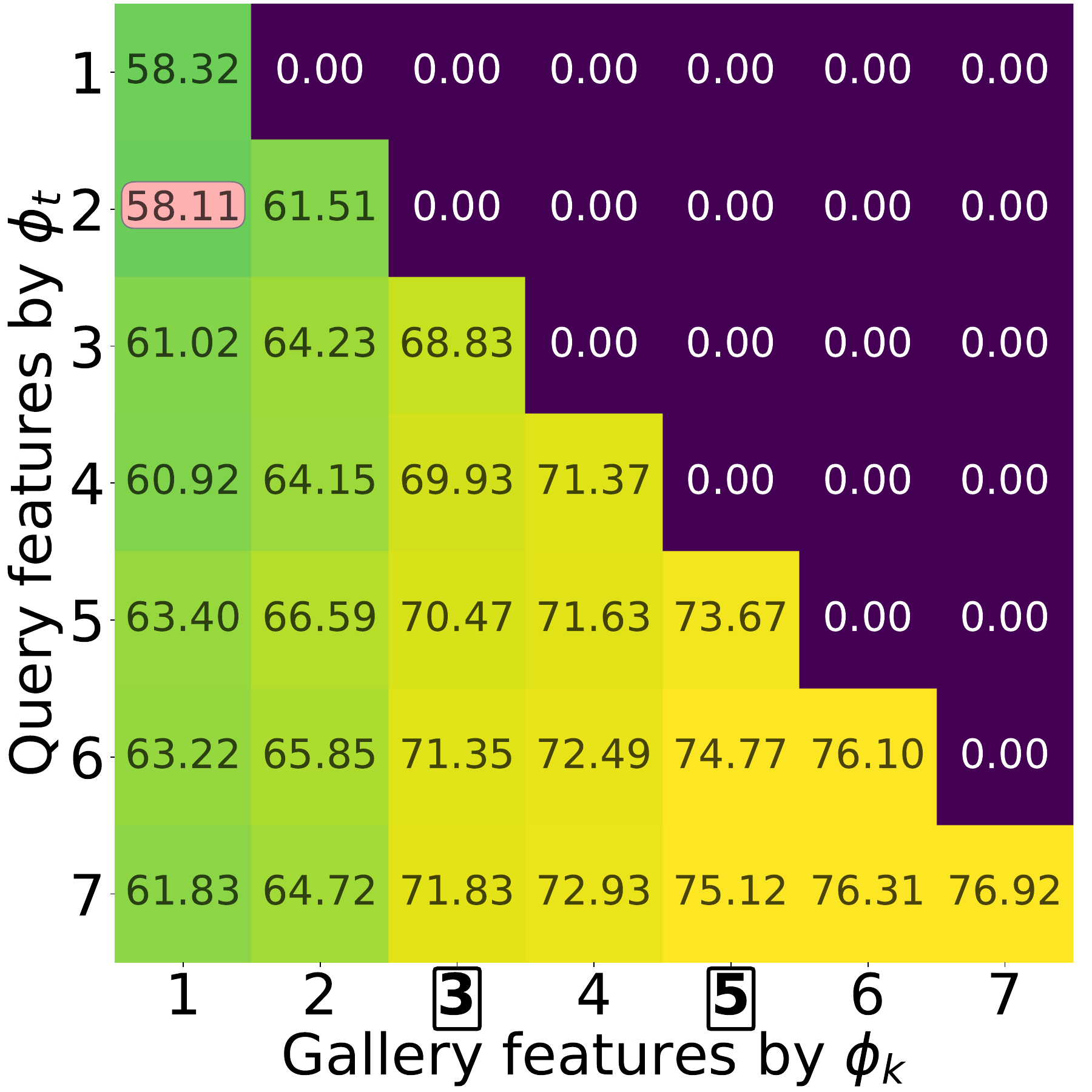}
        \caption{$d$-Simplex-HOC (in this paper)}
    \label{fig:our-10cifar-iamcl2r}
    \end{subfigure}
    \hspace{-16pt}
    \begin{subfigure}{0.08\linewidth}
        \centering
        \includegraphics[width=0.385\linewidth]{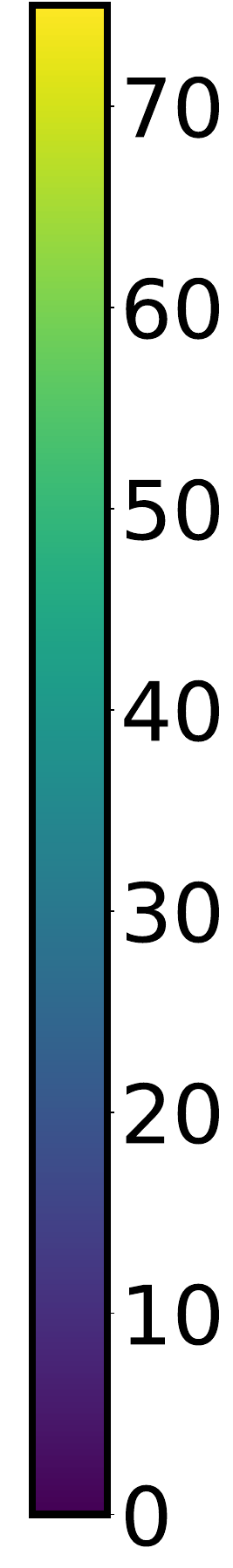}        
        \vspace{0.85cm}
    \end{subfigure}
    \caption{Compatibility Matrices for $d$-Simplex-HOC, CVS, BCT-ER, and $d$-Simplex-FD on CIFAR100R/10 across $7$ tasks. Model replacements at tasks $3$ and $5$ are highlighted in bold. Entries failing to meet compatibility criteria as defined in \cite{shen2020towards} are marked with a light-red background.
    \vspace{-0.3cm}
    }
    \label{fig:cifar-10cl2r-iamcl2r}
\end{figure*}

\noindent
\textbf{Fine-tuning.}
We replicate the fact that dataset size for training third-party models is typically significantly larger than the dataset size used for fine-tuning~\cite{ramasesh2021effect}.
According to this, pre-trained models are fine-tuned with a reduced version of CIFAR100~\cite{Krizhevsky2009LearningML} denoted in this paper as CIFAR100R. 

We considered two distinct task sequences consisting of 7 and 31 tasks each.
We fine-tune the pre-trained model with an initial task comprising 10 classes. Subsequently, for the sequences of 7 and 31 tasks, the respective tasks contain 15 and 3 classes each.
The fine-tuning process incorporates incoming task data, consisting of 300 images per class, and utilizes an episodic memory that stores 20 images from each class of previous tasks.

\noindent
\textbf{Model Replacement.}
In our experiments, we verify the impact of replacing the current fine-tuned model with two improved models pre-trained elsewhere. The two replacements occur while fine-tuning on CIFAR100R: at the third and fifth tasks in the shorter sequence, and at the eleventh and twenty-first tasks in the longer sequence. 
We also consider the challenging scenario of improved model replacement considering more sophisticated network architectures.

The $d$-Simplex fixed classifier is pre-allocated with a number of classes $K$, ensuring enough space to accommodate future classes for both pre-training and fine-tuning.
Class assignments for pre-training are made from left to right, and for fine-tuning, from right to left. This straightforward convention is used to ensure that classes assigned for pre-training and fine-tuning remain distinct, without overlap. 
Other non overlapping assignment methods could also be used.

\noindent
\textbf{Network Architectures.}
We use ResNet18~\cite{he2016deep} as network architecture.
In the scenario using more sophisticated network architectures, we initially replace ResNet18 with SENet18~\cite{hu2018squeeze}, followed by a subsequent replacement with a RegNetY\_400MF~\cite{radosavovic2020designing}.

\noindent
\textbf{Hyper-parameters.}
The ResNet18, SENet18, and RegNetY\_400MF models were pre-trained on ImageNet32 using the following hyper-parameters: 300 epochs, a batch size of 128, and an initial SGD optimizer learning rate of 0.1, which was adjusted using a Cosine Annealing schedule. 
For each task used for fine-tuning, the model is trained for 70 epochs with a batch size of 128, starting with a learning rate of 0.001 that was reduced by a factor of 10 after the 50th and 64th epochs.
The $d$-Simplex was pre-allocated with $K=1024$ classes (i.e., $d=K-1$).

\noindent
\textbf{Performance Evaluation.}
The evaluation focuses on the open-set recognition task, in which separated datasets for training and evaluation are required. The standard 1:N search protocol, applicable to re-identification and similar tasks \cite{shen2020towards}, is employed in the evaluation.
To ensure strict separation between datasets, the CIFAR10 dataset is utilized for evaluation during fine-tuning with CIFAR100R. Specifically, the test set of CIFAR10, comprising 10,000 images, is used as the gallery set, while its training set of 50,000 images serves as the query set.

Following \cite{shen2020towards} and \cite{biondi2023cores}, 
we measure performance progression across the two sequences of tasks using two established metrics: Average Compatibility ($AC$) and Average multi-model Accuracy (referred shortly as to $AA_t$). 
The metric $AC$ quantifies the extent of compatibility across all possible pairs of model combinations by providing a normalized count of times in which compatibility is achieved. Conversely, $AA_t$ calculates the mean accuracy across all combinations of the previously learned models until task $t$, providing an overall measure of accuracy.

\subsection{IAM-CL$^2$R: Comparative Results}
\label{sec:experimental_results}
We performed a comparative analysis of $d$-Simplex-HOC against FAN~\cite{iscen2020memory}, CVS~\cite{Wan_2022_CVPR}, $d$-Simplex-FD~\cite{biondi2023cl2r}, and the lifelong adapted versions of BCT~\cite{shen2020towards} (BCT-ER), LCE~\cite{meng2021learning} (LCE-ER), and AdvBCT~\cite{pan2023boundary} (AdvBCT-ER). The experiments also incorporate a baseline method, Experience Replay (ER), in which the model is fine-tuned using cross-entropy loss on data of the new task and an episodic memory. 
Ablation studies of IAM-CL$^2$R with the $d$-Simplex-HOC are provided in the Appendix.
\begin{figure}[t]
    \centering
    \begin{subfigure}{\linewidth}
        \hspace{8pt}\includegraphics[width=0.97\linewidth]{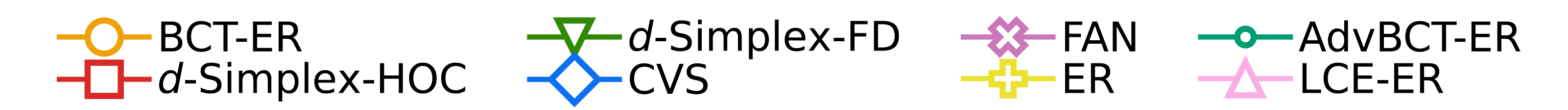}
    \end{subfigure}
    \hspace{-10pt}
    \includegraphics[width=.96\linewidth]{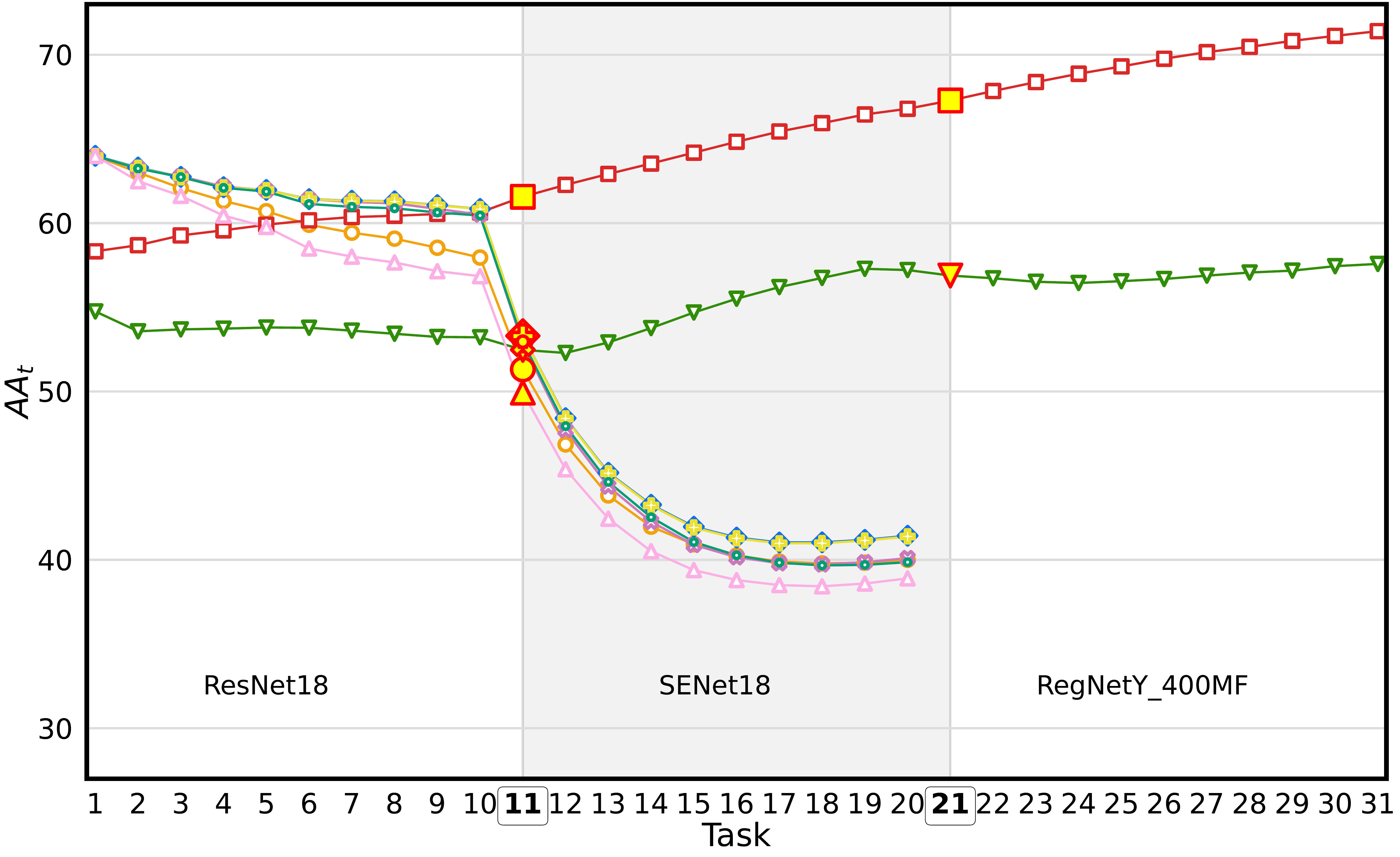}
    \caption{Plots of Average multi-model Accuracy ($AA_t$) for $31$ tasks on CIFAR100R/10, showing the impact of model replacements with different network architectures at tasks $11$ and $21$.
    }
    \label{fig:cifar31-iamcl2r-diff-arch}
\end{figure}

\noindent
\textbf{Replacing: Same Architecture, Expanded Data.} Fig.~\ref{fig:iamcl2r} presents the Average multi-model Accuracy at task $t$ ($AA_{t}$) for learning scenarios with model replacement as depicted in Fig.~\ref{fig:cifar30-iamcl2r} and for those without as depicted in Fig.~\ref{fig:cifar30-cl2r}. The experiment involves fine-tuning a ResNet18 model across $31$ tasks.
The comparison provides insights into the performance benefits that can be obtained by replacing models when representations are trained in a compatible manner.
The $d$-Simplex-HOC effectively incorporates improvements from model replacements, 
showing increased performance compared to the case without model replacement, as indicated in Fig.~\ref{fig:cifar30-cl2r}. The $d$-Simplex-FD demonstrates a similar capability, though to a reduced extent.
The other methods have a clear performance decay after model replacements and end up with a worse performance than the case without replacement. This can be attributed to the fact that after replacement, fine-tuning is applied to a model obtained by retraining the network from scratch, leading to an entirely different representation.

Further performance details, as indicated by the self and cross-test accuracy values, are shown according to the compatibility matrices \cite{shen2020towards,biondi2023cores}.
Fig.~\ref{fig:cifar-10cl2r-iamcl2r} shows these values for CVS, BCT-ER, $d$-Simplex-FD, and $d$-Simplex-HOC in the 7 tasks sequence. 
The values reveal that the $d$-Simplex-HOC effectively leverages the improved expressive power of the models after replacement, except in one instance. This exception, where the model is not compatible and the cross-test accuracy falls below the self-test accuracy, is shown in Fig.~\ref{fig:our-10cifar-iamcl2r}.
Both CVS and BCT-ER score near zero cross-tests accuracy after model replacements as indicated by the values in the blue sub matrix blocks shown in Fig.~\ref{fig:cvs-10cifar-iamcl2r} and Fig.~\ref{fig:bct-10cifar-iamcl2r}. This leads to mostly non-compatible representations.
Although both $d$-Simplex-HOC and $d$-Simplex-FD utilize the $d$-Simplex fixed classifier to learn stationary representations, the former shows better performance. This can be attributed to the high-order alignment achievable through the HOC loss.
To provide a full evaluation of compatibility, the $\mathit{AA}_t$ of Fig.~\ref{fig:iamcl2r} is complemented with the Average Compatibility $AC$ in Tab.~\ref{tab:iamcl2r_25tasks}. We also report the Average multi-model Accuracy $AA_7$ and $AA_{31}$  for methods compared at the end of the \mbox{$7$-th} and \mbox{$31$-th} task, respectively.  It is observed that, in both instances, all models---with the exception of $d$-Simplex-HOC---fail to achieve significant compatibility performance. 

\begin{table}[h]
\small
\centering 
\sisetup{detect-weight=true,detect-all,table-format=-1.3,round-mode=places, round-precision=2}
    \begin{tabular}{l SSSS}
        \toprule 
        \multirow{4}{*}{\shortstack{\textsc{Method}}}
        & \multicolumn{2}{c}{\textcolor{black}{7 tasks}} & \multicolumn{2}{c}{\textcolor{black}{31 tasks}} \\
        \cmidrule(lr){2-3}  \cmidrule(lr){4-5} 
        
        &
        {$AC$} &
        {$AA_7$ } &
        {$AC$} &
        {$AA_{31}$ }
        \\ 
        \midrule   
    ER baseline &$\times$& 36.217 & \hspace{-7pt}{$<$0.01} & 31.299 \\
    FAN~\cite{iscen2020memory}  &$\times$& 36.316 & \hspace{-7pt}{$<$0.01}  & 30.787 \\
    BCT-ER~\cite{shen2020towards}  & $\times$ & 35.59 &$\times$& 29.878  \\ 
    LCE-ER~\cite{meng2021learning}  &$\times$& 34.891&$\times$& 29.299 \\
    AdvBCT-ER~\cite{pan2023boundary} & $\times$& 35.732 &$\times$& 30.097 \\
    CVS~\cite{Wan_2022_CVPR}  &  $\times$ & 36.306 & 0.006452 & 31.344 \\
    $d$-Simplex-FD~\cite{biondi2023cl2r} & 0.04762 & 56.581  & 0.2108 & 56.267 \\
    $d$-Simplex-HOC & \textbf{0.95} & \hspace{-3pt}\textbf{68.13} &  \textbf{0.65} & \hspace{-3pt}\textbf{67.40} \\
        
        \bottomrule
    \end{tabular}  
    \captionof{table}{Compatibility metrics with CIFAR100R/10 for $7$ tasks with model replacements at task $3$ and task $5$, and $31$ tasks with model replacements at task $11$ and task $21$. ``$\times$'' indicates the case in which compatibility is not achieved.
    }
    \label{tab:iamcl2r_25tasks}
\end{table}

\noindent
\textbf{Replacing: Different Architectures, Expanded Data.}
Fig.~\ref{fig:cifar31-iamcl2r-diff-arch} shows the performance of the evaluated methods when the original ResNet18 is replaced first by a SENet18 and then by a more expressive RegNetY\_400MF. It is observed that the change of network architecture not only does not adversely affect compatibility in the $d$-Simplex-HOC but takes advantage of their more expressive representation power. 
In particular, direct comparison of Fig.~\ref{fig:cifar31-iamcl2r-diff-arch} with Fig.~\ref{fig:cifar-10cl2r-iamcl2r} shows that $d$-Simplex-HOC improves performance gradually with each model replacement. This is in contrast to $d$-Simplex-FD, which does not demonstrate the same trends leading to a plateau around the 20-th task.
Given the different feature sizes before and after the second model replacement with the RegNetY\_400MF architecture---$512$ and $384$, respectively---all methods except $d$-Simplex-HOC and $d$-Simplex-FD 
require non-trivial extensions to adapt to the changed feature size. According to this, for these methods, evaluation cannot be reported.

\section{Conclusion}
\label{sec:conclusion}
In this paper, we have investigated the concept of learning compatible representations through the principle of stationarity. We demonstrated that stationary representations optimally approximate compatibility according to its definition.
We demonstrated that better model alignment through higher-order dependencies can be obtained by training with a loss derived from one of the compatibility inequality constraints.
Finally, empirical evidence confirmed that stationary representations enable uninterrupted retrieval service allowing for fine-tuning and model replacement to occur concurrently and asynchronously with limited interference.

\noindent
\textbf{Acknowledgment:}
This work was partially supported by the European Commission under European Horizon 2020 Programme, grant number 951911 - AI4Media.
\\
\indent
We acknowledge the CINECA award under the ISCRA initiative, for the availability of high-performance computing resources and support (ISCRA-C ID:~HP10C4TIIM).

{\small
\bibliographystyle{unsrt}
\bibliography{11_references}

\begin{thebibliography}{10}

\bibitem{chopra2005learning}
Sumit Chopra, Raia Hadsell, and Yann LeCun.
\newblock Learning a similarity metric discriminatively, with application to face verification.
\newblock In {\em 2005 IEEE Computer Society Conference on Computer Vision and Pattern Recognition (CVPR'05)}, volume~1, pages 539--546. IEEE, 2005.

\bibitem{bengio2013representation}
Yoshua Bengio, Aaron Courville, and Pascal Vincent.
\newblock Representation learning: A review and new perspectives.
\newblock {\em IEEE transactions on pattern analysis and machine intelligence}, 35(8):1798--1828, 2013.

\bibitem{sharif2014cnn}
Ali Sharif~Razavian, Hossein Azizpour, Josephine Sullivan, and Stefan Carlsson.
\newblock Cnn features off-the-shelf: an astounding baseline for recognition.
\newblock In {\em Proceedings of the IEEE conference on computer vision and pattern recognition workshops}, pages 806--813, 2014.

\bibitem{YosinskiNIPS2014}
Jason Yosinski, Jeff Clune, Yoshua Bengio, and Hod Lipson.
\newblock How transferable are features in deep neural networks?
\newblock In Z.~Ghahramani, M.~Welling, C.~Cortes, N.~Lawrence, and K.~Q. Weinberger, editors, {\em Advances in Neural Information Processing Systems}, volume~27. Curran Associates, Inc., 2014.

\bibitem{taigman2014deepface}
Yaniv Taigman, Ming Yang, Marc'Aurelio Ranzato, and Lior Wolf.
\newblock Deepface: Closing the gap to human-level performance in face verification.
\newblock In {\em Proceedings of the IEEE conference on computer vision and pattern recognition}, pages 1701--1708, 2014.

\bibitem{sun2015deepid3}
Yi~Sun, Ding Liang, Xiaogang Wang, and Xiaoou Tang.
\newblock {DeepID3}: Face recognition with very deep neural networks.
\newblock {\em arXiv preprint arXiv:1502.00873}, 2015.

\bibitem{Ranjan2017}
Rajeev Ranjan, Carlos~D Castillo, and Rama Chellappa.
\newblock L2-constrained softmax loss for discriminative face verification.
\newblock {\em arXiv preprint arXiv:1703.09507}, 2017.

\bibitem{DBLP:conf/cvpr/DengGXZ19}
Jiankang Deng, Jia Guo, Niannan Xue, and Stefanos Zafeiriou.
\newblock Arcface: Additive angular margin loss for deep face recognition.
\newblock In {\em Proceedings of the IEEE/CVF Conference on Computer Vision and Pattern Recognition}, pages 4690--4699, 2019.

\bibitem{meng2021magface}
Qiang Meng, Shichao Zhao, Zhida Huang, and Feng Zhou.
\newblock {MagFace}: A universal representation for face recognition and quality assessment.
\newblock In {\em Proceedings of the IEEE/CVF Conference on Computer Vision and Pattern Recognition}, 2021.

\bibitem{Sun2018Beyond}
Yifan Sun, Liang Zheng, Yi~Yang, Qi~Tian, and Shengjin Wang.
\newblock Beyond part models: Person retrieval with refined part pooling (and a strong convolutional baseline).
\newblock In {\em European Conference on Computer Vision}, 2018.

\bibitem{Hermans2017In}
Alexander Hermans, Lucas Beyer, and Bastian Leibe.
\newblock In defense of the triplet loss for person re-identification.
\newblock {\em arXiv preprint arXiv:1703.07737}, 2017.

\bibitem{ristani2018features}
Ergys Ristani and Carlo Tomasi.
\newblock Features for multi-target multi-camera tracking and re-identification.
\newblock In {\em European Conference on Computer Vision}, pages 6036--6046, 2018.

\bibitem{radenovic2018revisiting}
Filip Radenovi{\'c}, Ahmet Iscen, Giorgos Tolias, Yannis Avrithis, and Ond{\v{r}}ej Chum.
\newblock Revisiting oxford and paris: Large-scale image retrieval benchmarking.
\newblock In {\em Proceedings of the IEEE conference on computer vision and pattern recognition}, pages 5706--5715, 2018.

\bibitem{radenovic2018fine}
Filip Radenovi{\'c}, Giorgos Tolias, and Ond{\v{r}}ej Chum.
\newblock Fine-tuning cnn image retrieval with no human annotation.
\newblock {\em IEEE transactions on pattern analysis and machine intelligence}, 41(7):1655--1668, 2018.

\bibitem{chen2022deep}
Wei Chen, Yu~Liu, Weiping Wang, Erwin~M Bakker, Theodoros Georgiou, Paul Fieguth, Li~Liu, and Michael~S Lew.
\newblock Deep learning for instance retrieval: A survey.
\newblock {\em IEEE Transactions on Pattern Analysis and Machine Intelligence}, 2022.

\bibitem{brown2020language}
Tom Brown, Benjamin Mann, Nick Ryder, Melanie Subbiah, Jared~D Kaplan, Prafulla Dhariwal, Arvind Neelakantan, Pranav Shyam, Girish Sastry, Amanda Askell, et~al.
\newblock Language models are few-shot learners.
\newblock {\em Advances in neural information processing systems}, 33:1877--1901, 2020.

\bibitem{radford2021learning}
Alec Radford, Jong~Wook Kim, Chris Hallacy, Aditya Ramesh, Gabriel Goh, Sandhini Agarwal, Girish Sastry, Amanda Askell, Pamela Mishkin, Jack Clark, et~al.
\newblock Learning transferable visual models from natural language supervision.
\newblock In {\em International Conference on Machine Learning}, pages 8748--8763. PMLR, 2021.

\bibitem{caccia2021new}
Lucas Caccia, Rahaf Aljundi, Nader Asadi, Tinne Tuytelaars, Joelle Pineau, and Eugene Belilovsky.
\newblock New insights on reducing abrupt representation change in online continual learning.
\newblock In {\em International Conference on Learning Representations}, 2021.

\bibitem{Davari_2022_CVPR}
MohammadReza Davari, Nader Asadi, Sudhir Mudur, Rahaf Aljundi, and Eugene Belilovsky.
\newblock Probing representation forgetting in supervised and unsupervised continual learning.
\newblock In {\em Proceedings of the IEEE/CVF Conference on Computer Vision and Pattern Recognition (CVPR)}, pages 16712--16721, June 2022.

\bibitem{barletti2022contrastive}
Tommaso Barletti, Niccolo' Biondi, Federico Pernici, Matteo Bruni, and Alberto Del~Bimbo.
\newblock Contrastive supervised distillation for continual representation learning.
\newblock {\em International Conference on Image Analysis and Processing}, 2022.

\bibitem{asadi2023prototype}
Nader Asadi, MohammadReza Davari, Sudhir Mudur, Rahaf Aljundi, and Eugene Belilovsky.
\newblock Prototype-sample relation distillation: towards replay-free continual learning.
\newblock In {\em International Conference on Machine Learning}, pages 1093--1106. PMLR, 2023.

\bibitem{wolf2019huggingface}
Thomas Wolf, Lysandre Debut, Victor Sanh, Julien Chaumond, Clement Delangue, Anthony Moi, Pierric Cistac, Tim Rault, R{\'e}mi Louf, Morgan Funtowicz, et~al.
\newblock Huggingface's transformers: State-of-the-art natural language processing.
\newblock {\em arXiv preprint arXiv:1910.03771}, 2019.

\bibitem{touvron2023llama}
Hugo Touvron, Thibaut Lavril, Gautier Izacard, Xavier Martinet, Marie-Anne Lachaux, Timoth{\'e}e Lacroix, Baptiste Rozi{\`e}re, Naman Goyal, Eric Hambro, Faisal Azhar, et~al.
\newblock Llama: Open and efficient foundation language models.
\newblock {\em arXiv preprint arXiv:2302.13971}, 2023.

\bibitem{raffel2023building}
Colin Raffel.
\newblock Building machine learning models like open source software.
\newblock {\em Communications of the ACM}, 66(2):38--40, 2023.

\bibitem{bommasani2021opportunities}
Rishi Bommasani, Drew~A Hudson, Ehsan Adeli, Russ Altman, Simran Arora, Sydney von Arx, Michael~S Bernstein, Jeannette Bohg, Antoine Bosselut, Emma Brunskill, et~al.
\newblock On the opportunities and risks of foundation models.
\newblock {\em arXiv preprint arXiv:2108.07258}, 2021.

\bibitem{sorscher2022beyond}
Ben Sorscher, Robert Geirhos, Shashank Shekhar, Surya Ganguli, and Ari~S Morcos.
\newblock Beyond neural scaling laws: beating power law scaling via data pruning.
\newblock {\em arXiv preprint arXiv:2206.14486}, 2022.

\bibitem{shen2020towards}
Yantao Shen, Yuanjun Xiong, Wei Xia, and Stefano Soatto.
\newblock Towards backward-compatible representation learning.
\newblock In {\em Proceedings of the IEEE/CVF Conference on Computer Vision and Pattern Recognition}, pages 6368--6377, 2020.

\bibitem{meng2021learning}
Qiang Meng, Chixiang Zhang, Xiaoqiang Xu, and Feng Zhou.
\newblock Learning compatible embeddings.
\newblock In {\em Proceedings of the IEEE/CVF International Conference on Computer Vision}, pages 9939--9948, 2021.

\bibitem{biondi2023cores}
Niccolo Biondi, Federico Pernici, Matteo Bruni, and Alberto Del~Bimbo.
\newblock Cores: Compatible representations via stationarity.
\newblock {\em IEEE Transactions on Pattern Analysis and Machine Intelligence}, 2023.

\bibitem{strubell2019energy}
Emma Strubell, Ananya Ganesh, and Andrew McCallum.
\newblock Energy and policy considerations for deep learning in nlp.
\newblock In {\em Proceedings of the 57th Annual Meeting of the Association for Computational Linguistics}, pages 3645--3650, 2019.

\bibitem{price2019privacy}
W~Nicholson Price and I~Glenn Cohen.
\newblock Privacy in the age of medical big data.
\newblock {\em Nature medicine}, 25(1):37--43, 2019.

\bibitem{lewis2020retrieval}
Patrick Lewis, Ethan Perez, Aleksandra Piktus, Fabio Petroni, Vladimir Karpukhin, Naman Goyal, Heinrich K{\"u}ttler, Mike Lewis, Wen-tau Yih, Tim Rockt{\"a}schel, et~al.
\newblock Retrieval-augmented generation for knowledge-intensive nlp tasks.
\newblock {\em Advances in Neural Information Processing Systems}, 33:9459--9474, 2020.

\bibitem{wang2020unified}
Chien{-}Yi Wang, Ya{-}Liang Chang, Shang{-}Ta Yang, Dong Chen, and Shang{-}Hong Lai.
\newblock Unified representation learning for cross model compatibility.
\newblock In {\em 31st British Machine Vision Conference 2020, {BMVC} 2020}. {BMVA} Press, 2020.

\bibitem{zhang2022towards}
Binjie Zhang, Yixiao Ge, Yantao Shen, Shupeng Su, Chun Yuan, Xuyuan Xu, Yexin Wang, and Ying Shan.
\newblock Towards universal backward-compatible representation learning.
\newblock {\em arXiv preprint arXiv:2203.01583}, 2022.

\bibitem{duggal2021compatibility}
Rahul Duggal, Hao Zhou, Shuo Yang, Yuanjun Xiong, Wei Xia, Zhuowen Tu, and Stefano Soatto.
\newblock Compatibility-aware heterogeneous visual search.
\newblock In {\em Proceedings of the IEEE/CVF Conference on Computer Vision and Pattern Recognition}, pages 10723--10732, 2021.

\bibitem{pan2023boundary}
Tan Pan, Furong Xu, Xudong Yang, Sifeng He, Chen Jiang, Qingpei Guo, Feng Qian, Xiaobo Zhang, Yuan Cheng, Lei Yang, et~al.
\newblock Boundary-aware backward-compatible representation via adversarial learning in image retrieval.
\newblock In {\em Proceedings of the IEEE/CVF Conference on Computer Vision and Pattern Recognition}, pages 15201--15210, 2023.

\bibitem{iscen2020memory}
Ahmet Iscen, Jeffrey Zhang, Svetlana Lazebnik, and Cordelia Schmid.
\newblock Memory-efficient incremental learning through feature adaptation.
\newblock In {\em European Conference on Computer Vision}, pages 699--715. Springer, 2020.

\bibitem{Wan_2022_CVPR}
Timmy S.~T. Wan, Jun-Cheng Chen, Tzer-Yi Wu, and Chu-Song Chen.
\newblock Continual learning for visual search with backward consistent feature embedding.
\newblock In {\em Proceedings of the IEEE/CVF Conference on Computer Vision and Pattern Recognition (CVPR)}, pages 16702--16711, June 2022.

\bibitem{biondi2023cl2r}
Niccolo Biondi, Federico Pernici, Matteo Bruni, Daniele Mugnai, and Alberto~Del Bimbo.
\newblock Cl2r: Compatible lifelong learning representations.
\newblock {\em ACM Transactions on Multimedia Computing, Communications and Applications}, 18(2s):1--22, 2023.

\bibitem{trauble2021backward}
Frederik Tr{\"a}uble, Julius Von~K{\"u}gelgen, Matth{\"a}us Kleindessner, Francesco Locatello, Bernhard Sch{\"o}lkopf, and Peter Gehler.
\newblock Backward-compatible prediction updates: A probabilistic approach.
\newblock {\em Advances in Neural Information Processing Systems}, 34:116--128, 2021.

\bibitem{zhou2023bt}
Yifei Zhou, Zilu Li, Abhinav Shrivastava, Hengshuang Zhao, Antonio Torralba, Taipeng Tian, and Ser-Nam Lim.
\newblock Bt\^{} 2: Backward-compatible training with basis transformation.
\newblock In {\em Proceedings of the IEEE/CVF International Conference on Computer Vision}, pages 11229--11238, 2023.

\bibitem{pernici2021regular}
Federico Pernici, Matteo Bruni, Claudio Baecchi, and Alberto Del~Bimbo.
\newblock Regular polytope networks.
\newblock {\em IEEE Transactions on Neural Networks and Learning Systems}, 2021.

\bibitem{Pernici_2019_CVPR_Workshops}
Federico Pernici, Matteo Bruni, Claudio Baecchi, and Alberto Del~Bimbo.
\newblock Maximally compact and separated features with regular polytope networks.
\newblock In {\em Proceedings of the IEEE/CVF Conference on Computer Vision and Pattern Recognition (CVPR) Workshops}, June 2019.

\bibitem{papyan2020prevalence}
Vardan Papyan, XY~Han, and David~L Donoho.
\newblock Prevalence of neural collapse during the terminal phase of deep learning training.
\newblock {\em Proceedings of the National Academy of Sciences}, 117(40):24652--24663, 2020.

\bibitem{oord2018representation}
Aaron van~den Oord, Yazhe Li, and Oriol Vinyals.
\newblock Representation learning with contrastive predictive coding.
\newblock {\em arXiv preprint arXiv:1807.03748}, 2018.

\bibitem{belkin2019reconciling}
Mikhail Belkin, Daniel Hsu, Siyuan Ma, and Soumik Mandal.
\newblock Reconciling modern machine-learning practice and the classical bias--variance trade-off.
\newblock {\em Proceedings of the National Academy of Sciences}, 116(32):15849--15854, 2019.

\bibitem{hoffer2018fixed}
Elad Hoffer, Itay Hubara, and Daniel Soudry.
\newblock Fix your classifier: the marginal value of training the last weight layer.
\newblock In {\em International Conference on Learning Representations}, 2018.

\bibitem{pernici2019fix}
Federico Pernici, Matteo Bruni, Claudio Baecchi, and Alberto Del~Bimbo.
\newblock Fix your features: Stationary and maximally discriminative embeddings using regular polytope (fixed classifier) networks.
\newblock {\em arXiv preprint arXiv:1902.10441}, 2019.

\bibitem{liu2018learning}
Weiyang Liu, Rongmei Lin, Zhen Liu, Lixin Liu, Zhiding Yu, Bo~Dai, and Le~Song.
\newblock Learning towards minimum hyperspherical energy.
\newblock {\em Advances in neural information processing systems}, 31, 2018.

\bibitem{pernici2021class}
Federico Pernici, Matteo Bruni, Claudio Baecchi, Francesco Turchini, and Alberto Del~Bimbo.
\newblock Class-incremental learning with pre-allocated fixed classifiers.
\newblock In {\em 2020 25th International Conference on Pattern Recognition (ICPR)}, pages 6259--6266. IEEE, 2021.

\bibitem{yang2022neural}
Yibo Yang, Haobo Yuan, Xiangtai Li, Zhouchen Lin, Philip Torr, and Dacheng Tao.
\newblock Neural collapse inspired feature-classifier alignment for few-shot class-incremental learning.
\newblock In {\em The Eleventh International Conference on Learning Representations}, 2022.

\bibitem{zhou2022forward}
Da-Wei Zhou, Fu-Yun Wang, Han-Jia Ye, Liang Ma, Shiliang Pu, and De-Chuan Zhan.
\newblock Forward compatible few-shot class-incremental learning.
\newblock In {\em Proceedings of the IEEE/CVF conference on computer vision and pattern recognition}, pages 9046--9056, 2022.

\bibitem{mixon2022neural}
Dustin~G Mixon, Hans Parshall, and Jianzong Pi.
\newblock Neural collapse with unconstrained features.
\newblock {\em Sampling Theory, Signal Processing, and Data Analysis}, 20(2):1--13, 2022.

\bibitem{fang2021exploring}
Cong Fang, Hangfeng He, Qi~Long, and Weijie~J Su.
\newblock Exploring deep neural networks via layer-peeled model: Minority collapse in imbalanced training.
\newblock {\em Proceedings of the National Academy of Sciences}, 118(43):e2103091118, 2021.

\bibitem{graf2021dissecting}
Florian Graf, Christoph Hofer, Marc Niethammer, and Roland Kwitt.
\newblock Dissecting supervised contrastive learning.
\newblock In {\em International Conference on Machine Learning}, pages 3821--3830. PMLR, 2021.

\bibitem{zhu2021geometric}
Zhihui Zhu, Tianyu Ding, Jinxin Zhou, Xiao Li, Chong You, Jeremias Sulam, and Qing Qu.
\newblock A geometric analysis of neural collapse with unconstrained features.
\newblock {\em Advances in Neural Information Processing Systems}, 34:29820--29834, 2021.

\bibitem{yang2022we}
Yibo Yang, Shixiang Chen, Xiangtai Li, Liang Xie, Zhouchen Lin, and Dacheng Tao.
\newblock Inducing neural collapse in imbalanced learning: Do we really need a learnable classifier at the end of deep neural network?
\newblock {\em Advances in Neural Information Processing Systems}, 35:37991--38002, 2022.

\bibitem{kothapalli2022neural}
Vignesh Kothapalli, Ebrahim Rasromani, and Vasudev Awatramani.
\newblock Neural collapse: A review on modelling principles and generalization.
\newblock {\em arXiv preprint arXiv:2206.04041}, 2022.

\bibitem{lenc2015understanding}
Karel Lenc and Andrea Vedaldi.
\newblock Understanding image representations by measuring their equivariance and equivalence.
\newblock In {\em Proceedings of the IEEE conference on computer vision and pattern recognition}, pages 991--999, 2015.

\bibitem{li2015convergent}
Yixuan Li, Jason Yosinski, Jeff Clune, Hod Lipson, and John Hopcroft.
\newblock Convergent learning: Do different neural networks learn the same representations?
\newblock In {\em Feature Extraction: Modern Questions and Challenges}, pages 196--212. PMLR, 2015.

\bibitem{wang2018towards}
Liwei Wang, Lunjia Hu, Jiayuan Gu, Zhiqiang Hu, Yue Wu, Kun He, and John Hopcroft.
\newblock Towards understanding learning representations: To what extent do different neural networks learn the same representation.
\newblock {\em Advances in neural information processing systems}, 31, 2018.

\bibitem{kornblith2019similarity}
Simon Kornblith, Mohammad Norouzi, Honglak Lee, and Geoffrey Hinton.
\newblock Similarity of neural network representations revisited.
\newblock In {\em International conference on machine learning}, pages 3519--3529. PMLR, 2019.

\bibitem{bansal2021revisiting}
Yamini Bansal, Preetum Nakkiran, and Boaz Barak.
\newblock Revisiting model stitching to compare neural representations.
\newblock {\em Advances in neural information processing systems}, 34:225--236, 2021.

\bibitem{Chen_2019_CVPR}
Ken Chen, Yichao Wu, Haoyu Qin, Ding Liang, Xuebo Liu, and Junjie Yan.
\newblock {R3} adversarial network for cross model face recognition.
\newblock In {\em {CVPR}}, pages 9868--9876. Computer Vision Foundation / {IEEE}, 2019.

\bibitem{hu2022learning}
Weihua Hu, Rajas Bansal, Kaidi Cao, Nikhil Rao, Karthik Subbian, and Jure Leskovec.
\newblock Learning backward compatible embeddings.
\newblock In {\em Proceedings of the 28th ACM SIGKDD Conference on Knowledge Discovery and Data Mining}, pages 3018--3028, 2022.

\bibitem{ramanujan2022forward}
Vivek Ramanujan, Pavan Kumar~Anasosalu Vasu, Ali Farhadi, Oncel Tuzel, and Hadi Pouransari.
\newblock Forward compatible training for large-scale embedding retrieval systems.
\newblock In {\em Proceedings of the IEEE/CVF Conference on Computer Vision and Pattern Recognition}, pages 19386--19395, 2022.

\bibitem{Liu2016large_rebuttal}
Weiyang~Liu et~al.
\newblock Large-margin softmax loss for convolutional neural networks.
\newblock {\em ICML}, 2016.

\bibitem{Liu2017sphere_rebuttal}
Weiyang~Liu et~al.
\newblock Sphereface: Deep hypersphere embedding for face recognition.
\newblock {\em CVPR}, 2017.

\bibitem{hestness2017deep}
Joel Hestness, Sharan Narang, Newsha Ardalani, Gregory Diamos, Heewoo Jun, Hassan Kianinejad, Md~Mostofa~Ali Patwary, Yang Yang, and Yanqi Zhou.
\newblock Deep learning scaling is predictable, empirically.
\newblock {\em arXiv preprint arXiv:1712.00409}, 2017.

\bibitem{prato2021scaling}
Gabriele Prato, Simon Guiroy, Ethan Caballero, Irina Rish, and Sarath Chandar.
\newblock Scaling laws for the out-of-distribution generalization of image classifiers.
\newblock {\em ICML 2021 Workshop on Uncertainty and Robustness in Deep Learning.}, 2021.

\bibitem{nakkiran2021deep}
Preetum Nakkiran, Gal Kaplun, Yamini Bansal, Tristan Yang, Boaz Barak, and Ilya Sutskever.
\newblock Deep double descent: Where bigger models and more data hurt.
\newblock {\em Journal of Statistical Mechanics: Theory and Experiment}, 2021(12):124003, 2021.

\bibitem{caballero2023broken}
Ethan Caballero, Kshitij Gupta, Irina Rish, and David Krueger.
\newblock Broken neural scaling laws.
\newblock In {\em The Eleventh International Conference on Learning Representations}, 2023.

\bibitem{burgstaller2009average}
Bernhard Burgstaller and Friedrich Pillichshammer.
\newblock The average distance between two points.
\newblock {\em Bulletin of the Australian Mathematical Society}, 80(3):353--359, 2009.

\bibitem{solomon1978geometric}
Herbert Solomon.
\newblock {\em Geometric probability}.
\newblock SIAM, 1978.

\bibitem{kendallgeometrical}
Maurice G. (Maurice~George) Kendall.
\newblock {\em Geometrical probability}.
\newblock Griffin's statistical monographs \& courses ; no. 10. C. Griffin, Hafner Pub. Co, London ; New York, 1963.

\bibitem{sors2004integral}
Luis Antonio~Santal{\'o} Sors and Luis~A Santal{\'o}.
\newblock {\em Integral geometry and geometric probability}.
\newblock Cambridge university press, 2004.

\bibitem{kirby1982average}
HR~Kirby and John~David Murchland.
\newblock {\em Average Distance Calculations Between and Within Zones: Some Issues at the Interface of Continous and Discrete Models}.
\newblock University of Leeds, Institute for Transport Studies, 1982.

\bibitem{vaughan1984approximate}
Rodney Vaughan.
\newblock Approximate formulas for average distances associated with zones.
\newblock {\em Transportation science}, 18(3):231--244, 1984.

\bibitem{fairthorne1964distances}
David Fairthorne.
\newblock The distances between random points in two concentric circles.
\newblock {\em Biometrika}, 51(1/2):275--277, 1964.

\bibitem{tian2019contrastive}
Yonglong Tian, Dilip Krishnan, and Phillip Isola.
\newblock Contrastive representation distillation.
\newblock In {\em International Conference on Learning Representations}, 2020.

\bibitem{yang2023neural}
Yibo Yang, Haobo Yuan, Xiangtai Li, Zhouchen Lin, Philip Torr, and Dacheng Tao.
\newblock Neural collapse inspired feature-classifier alignment for few-shot class-incremental learning.
\newblock In {\em ICLR}, 2023.

\bibitem{jiang2021churn}
Heinrich Jiang, Harikrishna Narasimhan, Dara Bahri, Andrew Cotter, and Afshin Rostamizadeh.
\newblock Churn reduction via distillation.
\newblock In {\em International Conference on Learning Representations}, 2021.

\bibitem{tian2020contrastive}
Yonglong Tian, Dilip Krishnan, and Phillip Isola.
\newblock Contrastive multiview coding.
\newblock In {\em Computer Vision--ECCV 2020: 16th European Conference, Glasgow, UK, August 23--28, 2020, Proceedings, Part XI 16}, pages 776--794. Springer, 2020.

\bibitem{wen2016discriminative}
Yandong Wen, Kaipeng Zhang, Zhifeng Li, and Yu~Qiao.
\newblock A discriminative feature learning approach for deep face recognition.
\newblock In {\em European Conference on Computer Vision}, pages 499--515. Springer, 2016.

\bibitem{chrabaszcz2017downsampled}
Patryk Chrabaszcz, Ilya Loshchilov, and Frank Hutter.
\newblock A downsampled variant of imagenet as an alternative to the cifar datasets.
\newblock {\em arXiv preprint arXiv:1707.08819}, 2017.

\bibitem{ramasesh2021effect}
Vinay~Venkatesh Ramasesh, Aitor Lewkowycz, and Ethan Dyer.
\newblock Effect of scale on catastrophic forgetting in neural networks.
\newblock In {\em International Conference on Learning Representations}, 2021.

\bibitem{Krizhevsky2009LearningML}
A.~Krizhevsky.
\newblock {Learning Multiple Layers of Features from Tiny Images}.
\newblock Technical report, Univ. Toronto, 2009.

\bibitem{he2016deep}
Kaiming He, Xiangyu Zhang, Shaoqing Ren, and Jian Sun.
\newblock Deep residual learning for image recognition.
\newblock In {\em Proceedings of the IEEE conference on computer vision and pattern recognition}, pages 770--778, 2016.

\bibitem{hu2018squeeze}
Jie Hu, Li~Shen, and Gang Sun.
\newblock Squeeze-and-excitation networks.
\newblock In {\em Proceedings of the IEEE conference on computer vision and pattern recognition}, pages 7132--7141, 2018.

\bibitem{radosavovic2020designing}
Ilija Radosavovic, Raj~Prateek Kosaraju, Ross Girshick, Kaiming He, and Piotr Doll{\'a}r.
\newblock Designing network design spaces.
\newblock In {\em Proceedings of the IEEE/CVF conference on computer vision and pattern recognition}, pages 10428--10436, 2020.

\bibitem{brauchart2018random}
Johann~S Brauchart, Alexander~B Reznikov, Edward~B Saff, Ian~H Sloan, Yu~Guang Wang, and Robert~S Womersley.
\newblock Random point sets on the sphere—hole radii, covering, and separation.
\newblock {\em Experimental Mathematics}, 27(1):62--81, 2018.

\bibitem{chen2020simple}
Ting Chen, Simon Kornblith, Mohammad Norouzi, and Geoffrey Hinton.
\newblock A simple framework for contrastive learning of visual representations.
\newblock In {\em International conference on machine learning}, pages 1597--1607. PMLR, 2020.

\bibitem{icarl}
Sylvestre-Alvise Rebuffi, Alexander Kolesnikov, Georg Sperl, and Christoph~H Lampert.
\newblock icarl: Incremental classifier and representation learning.
\newblock In {\em Proceedings of the IEEE conference on Computer Vision and Pattern Recognition}, pages 2001--2010, 2017.

\bibitem{kasarla2022maximum}
Tejaswi Kasarla, Gertjan~J Burghouts, Max van Spengler, Elise van~der Pol, Rita Cucchiara, and Pascal Mettes.
\newblock Maximum separation as inductive bias in one matrix.
\newblock In {\em NeurIPS}, 2022.

\end{thebibliography}
}

\ifarxiv \clearpage \appendix \appendix

\section{Stationarity-Compatibility Theorem}\label{sec:proof-app}

Before proceeding to the main theorem, a key Lemma is established. This Lemma, concerning the probability of a random point on a surface cap of a hypersphere, plays an essential role in the subsequent discussion.

\begin{lem}
\label{proposition_prob}
Let $\mathbf{w}_i \in \mathbb{R}^d$ for $i=1, \ldots, n$ be i.i.d. vectors from the uniform distribution on the unit hypersphere. Then the probability $P_{n, d}$ of a random vector on a hypersphere cap around $\mathbf{w}_i$ is given by:
\begin{equation}
P_{n, d} = \frac{1}{\sqrt{\pi}} \cdot \sin\left( \theta_{n,d} \right)^{d - 2} \cdot \frac{\Gamma\left( \frac{d}{2} \right)}{\Gamma\left( \frac{d}{2} - \frac{1}{2} \right)}
\label{eq:p_kd}
\end{equation}
where $\theta_{n, d}$ is the expected angle from a vector $\mathbf{w}_i$ to its nearest neighbor.
\end{lem}
\begin{proof}
We begin by noting that the probability $P$ of a random point on a hypersphere cap around prototype $\mathbf{w}_i$ is given by the ratio of the cap surface to the hypersphere's surface area. This can be approximated as $P=\frac{A_{\text{disc}}}{A}$ where $A_{\text{disc}}$ is the area of the disc locally approximating the cap around the prototype $\mathbf{w}_i$. 
The surface area $A$ of a hypersphere in $d$ dimensions is given by
\[
A=2 \pi^{d / 2} \frac{R^{d-1}}{\Gamma(d / 2)}
\]
and the hyperarea $A_{\text{disc}}$ of the disc is
\[
A_{\text{disc}}=2 \pi^{(d-1) / 2} \frac{r^{d-2}}{\Gamma((d-1) / 2)}.
\]
This leads to the simplified expression for the probability $P$ of a random point on a disc on a hypersphere:
\begin{equation}
P=\frac{r^{d-2} \cdot R^{1-d} \cdot \Gamma\left(\frac{d}{2}\right)}{\sqrt{\pi} \cdot \Gamma\left(\frac{d}{2}-\frac{1}{2}\right)}.
\label{eq:probability}
\end{equation}
Where $r$ is the radius of the surface disc (locally approximating the cap), $R$ is the radius of the hypersphere, $d$ is the number of dimensions, and $\Gamma$ is the gamma function.
Using spherical coordinates, the relationship between $R, r$, and the polar angle $\theta$ is $r=R\sin(\theta)$. We use $\theta_{n, d}$ as described in \cite{brauchart2018random} and \cite{DBLP:conf/cvpr/DengGXZ19} to denote the dependencies on $n$ and $d$:
\begin{equation}
\theta_{n, d}=n^{-\frac{2}{d-1}} \Gamma\left(1+\frac{1}{d-1}\right)\left(\frac{\Gamma\left(\frac{d}{2}\right)}{2 \sqrt{\pi}(d-1) \Gamma\left(\frac{d-1}{2}\right)}\right)^{-\frac{1}{d-1}}.
\label{eq:theta_kd}
\end{equation}
Substituting $r = R \sin(\theta_{n,d})$ into the probability $P$ of Eq.~\ref{eq:probability} and, considering  the unit hypersphere $R=1$, we get Eq.~\ref{eq:p_kd}.
This highlights the dependencies of the probability on both the number of prototypes $n$ and their dimension $d$.

\end{proof}

Lemma~\ref{proposition_prob} is used to demonstrate Theorem~\ref{theo:main} that is reported in the following for better comprehension. It is noteworthy that a disc in high dimensional space can be considered a hyperball when referring to its filled volume.

\begin{customthm}{1}[Stationarity $\implies$ Compatibility] 
\label{theo:main}
Let $\mathbf{W}=[ \mathbf{w}_1, \mathbf{w}_2, \ldots, \mathbf{w}_K ]$ be the $d \times K$ matrix of a $d$-Simplex fixed classifier.
Given two tasks, \( \mathcal{T}_k \) and \( \mathcal{T}_t \). The task \( \mathcal{T}_t \) is derived from \( \mathcal{T}_k \) by incorporating an additional training set \( \Delta\mathcal{T} \), such that \( \mathcal{T}_t = \mathcal{T}_k \cup \Delta\mathcal{T} \). The combined task, \( \mathcal{T}_t \), comprises a set of classes each denoted by $y$, where \( {y} \in \{ 1,2,\dots,K_t \} \) and \( K_t < K \).
Under the assumption that learning the new task \( \mathcal{T}_t \) causes the hyperball \( \mathcal{B}_k(\mathbf{w}_y) \) with radius \( r_k^y \) to shrink into a smaller hyperball \( \mathcal{B}_t(\mathbf{w}_y) \), i.e., \( r_{t}^y \leq r_k^y \) for all \( y \) in the set \( \{ 1,2, \dots, K_k \} \), then it necessarily follows that \( \phi_{t} \) and \( \phi_{k} \) optimally approximate the compatibility inequality constraints as defined in Def.~\ref{def:compatibility-shen} in expectation.
\end{customthm}

\begin{proof}
Let \( \phi_t(\mathbf{x}) \) and \( \phi_k(\mathbf{x}) \) be random variables representing the learned representations up to the \( t \)-th and the \( k \)-th task, respectively. We assume that these variables are distributed within hyperballs denoted as $\mathcal{B}_t(\mathbf{w}_y)$ and $\mathcal{B}_k(\mathbf{w}_y)$, where $y$ is a generic class label, according to the joint probability density function $f_{\phi_t(\mathbf{x}),\phi_k(\mathbf{x})}$.
Hyperballs are centered at the $d$-Simplex classifier prototype $\mathbf{w}_y$ and are defined as:
\begin{align}
\label{eq:ball_t}
\mathcal{B}_t(\mathbf{w}_y) &= \{\phi_t(\mathbf{x}) \in \mathbb{R}^d: ||\phi_t(\mathbf{x})-\mathbf{w}_y||_2 \leq r_t^y \},\\
\label{eq:ball_k}
\mathcal{B}_k(\mathbf{w}_y) &= \{\phi_k(\mathbf{x}) \in \mathbb{R}^d: ||\phi_k(\mathbf{x})-\mathbf{w}_y||_2 \leq r_k^y \}
\end{align}
being $r_t^y$ and $r_k^y$ the radii of $\mathcal{B}_t(\mathbf{w}_y)$ and $\mathcal{B}_k(\mathbf{w}_y)$, respectively.
The distance between the two random variables $\phi_t(\mathbf{x}_a)$ and  $\phi_k(\mathbf{x}_b)$ is a new random variable:
\begin{equation}
D_{k,t} =||\phi_t(\mathbf{x}_a)-\phi_k(\mathbf{x}_b)||.
\label{eq:ball_dist}
\end{equation}
\noindent
Verification in expectation of the compatibility definition of Def.~\ref{def:compatibility-shen} requires the evaluation of $D_{k,t}$, i.e, \mbox{ $\mathbb{E}[||\phi_k(\mathbf{x}_a)-\phi_t(\mathbf{x}_b)||]$ }, and compare it with the expected value of  $D_{k,k}$, \mbox{i.e., $\mathbb{E}[||\phi_k(\mathbf{x}_a)-\phi_k(\mathbf{x}_b)||]$}. 
Defining 
the function $g$ as:
$$
g\left(x_a, x_b\right)=||x_a-x_b||,
$$
the expected value $\mathbb{E}[D_{k,t}]$ of Eq.~\ref{eq:ball_dist} is given by:
\begin{equation}
\label{eq:distance_MC_ind}
\mathbb{E}[D_{k,t}]=\int\displaylimits_{\mathcal{B}_k^{y_i}} \int\displaylimits_{\mathcal{B}_t^{y_j}} g\left(x_a, x_b\right) f_{\phi_{k}, \phi_{t}}\left(x_a, x_b\right)dV( x_a) dV( x_b )
\end{equation}
where $y_i$ and $y_j$ denote the classes associated with $x_a$ and $x_a$, respectively, and $\mathcal{B}_k^{y_i}$, $\mathcal{B}_t^{y_j}$, and $f_{\phi_t, \phi_k}$ are simplified notations for  $\mathcal{B}_k(\mathbf{w}_{y_i})$, $\mathcal{B}_t(\mathbf{w}_{y_j})$, and $f_{\phi_t (\mathbf{x}), \phi_k (\mathbf{x})}$, respectively.

Eq.~\ref{eq:distance_MC_ind} is evaluated under the following assumptions: 
(1) UFM \cite{mixon2022neural}, which allows features of a model to be considered independent.
(2) The hypothesis of a $d$-Simplex fixed classifier. This assumption allows focusing on a single pairwise class interaction, as interactions with all other classes are symmetrically similar and fixed.
(3) Since $\phi_t(\mathbf{x})$ and $\phi_k(\mathbf{x})$ are derived from training two separate models, they are treated as independent random variables, each distributed according to $f_{\phi_t(\mathbf{x})}$ and $f_{\phi_k(\mathbf{x})}$, respectively.
As a consequence, the joint probability density function can be substituted by the product of the probability density functions of $\phi_{k}(\mathbf{x})$ and $\phi_{t}(\mathbf{x})$, i.e., 
$f_{\phi_{k}(\mathbf{x}), \phi_{t}(\mathbf{x})}\left(x_a, x_b\right)=f_{ \phi_{k}(\mathbf{x}) }\left(x_a\right) f_{\phi_{t}(\mathbf{x})}\left(x_b\right)$ and integral of Eq.~\ref{eq:distance_MC_ind} reduces to:
\begin{equation}
\label{eq:distance_MC}
\mathbb{E}[D_{k,t}] = 
\int\displaylimits_{\mathcal{B}_k^{y_j}} \int\displaylimits_{\mathcal{B}_t^{y_i}} ||x_a-x_b|| f_{ \phi_{k} }\left(x_a\right) f_{\phi_{t}}\left(x_b\right) dV( x_a) dV( x_b ).
\end{equation}
Lemma~\ref{proposition_prob} allows for the case-by-case evaluation of Equation~\ref{eq:distance_MC} in the case of assessing the alignment and compatibility of class prototypes in trainable and non-trainable classifiers.
From the Lemma it follows that when retraining a model from scratch in which the classifier is trainable, the probability of class prototypes falling, according to Nearest Neighbor rule, within their corresponding hyperballs of a previously trained model decreases exponentially as both dimensionality and the number of classes for training increases (Fig.~\ref{fig:pdisc-label}). 
Following the definition of Eq.~\ref{eq:first}, the conditions for optimal compatibility between prototypes of corresponding classes in both models are realized when their distance reaches its minimum value.
This occurs when they are perfectly aligned. In this case, classes will not manifest randomly and the probability of them falling within the same regions does not decrease exponentially.

Eq.~\ref{eq:theta_kd} in Lemma~\ref{proposition_prob}, also indicates that the introduction of new classes results in a decrease in the angles between them, a phenomenon also shown in \cite{DBLP:conf/cvpr/DengGXZ19}. Assuming two perfectly aligned models, the introduction of new classes in one of the models results in two effects: a decrease in intraclass and interclass distances between features. 
Such reductions in distance indicate a deviation from the concentric arrangement between of corresponding class hyperballs in the two models, leading to a compromise of the conditions for optimal compatibility.
While one might consider pre-allocating a large number of classes to leverage a broader representation space for future classes to prevent the reduction of class angles, this strategy is found to be suboptimal in trainable classifiers. In fact, without supervision, the pre-allocated prototypes for future classes tend to collapse onto each other, as evidenced by \cite{fang2021exploring,yang2022we}. This tendency illustrates the inherent limitations of this approach in achieving optimal compatibility with trainable classifiers.
\begin{figure}
    \centering
    \includegraphics[width=1.0\linewidth]{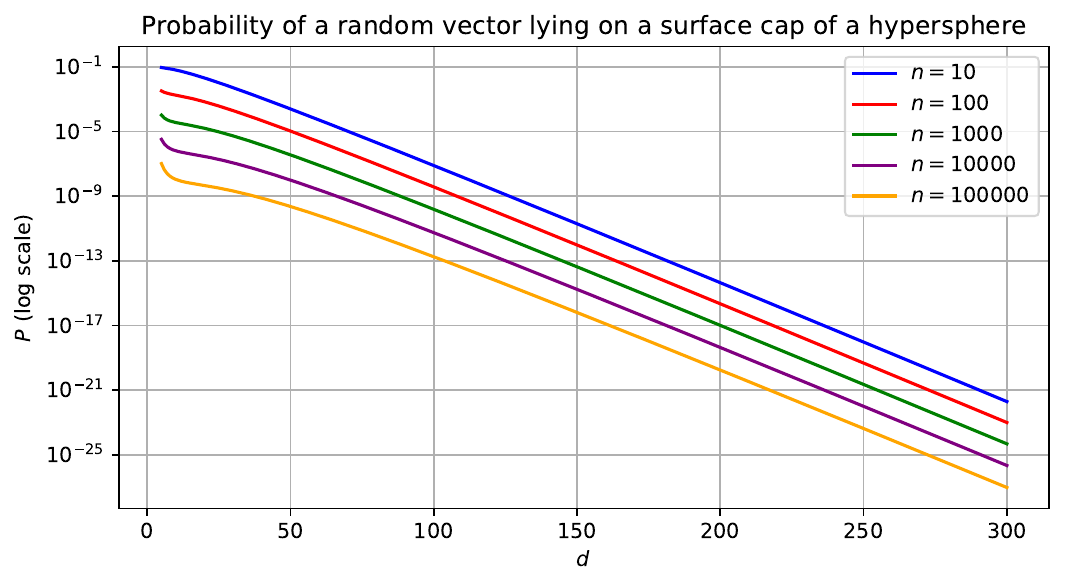}
    \caption{The probability $P$ of Eq.~\ref{eq:p_kd} of a point lying within a disc on a hypersphere's surface. Different curves (logarithmic scale) correspond to varying numbers of points sampled ($n$), across a dimension range ($d$). The plot shows that as the dimension and the number of points increases, the probability decreases significantly, reflecting the curse of dimensionality.}
    \label{fig:pdisc-label}
\end{figure}
\begin{figure}
    \centering
    \includegraphics[width=\linewidth]{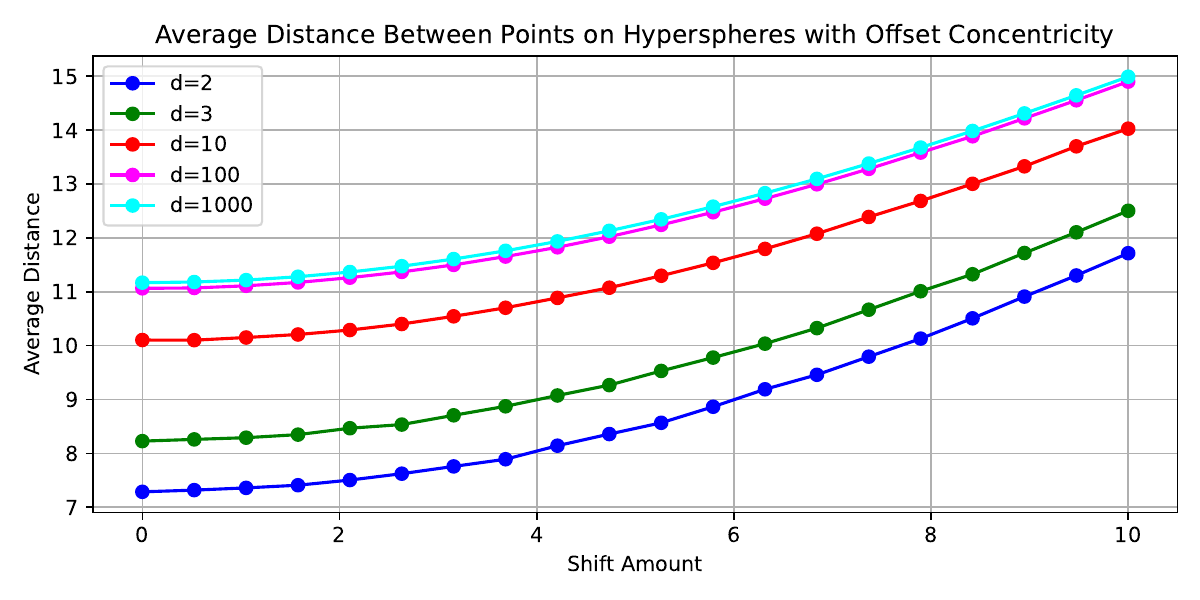}
    \caption{Expected distance of Eq. \ref{eq:distance_MC} between points on two closely aligned (or nearly concentric) hyperballs. Distance increases by shifting one of the hyperballs showing that optimality (i.e. less distance variation) is when hyperball are concentric.}
    \label{fig:pdisc-shift}
\end{figure}

\begin{figure*}[!ht]
\centering
   \subcaptionbox{Same class \label{fig:distance_sameclass}}{
        \includegraphics[height=4.7cm]{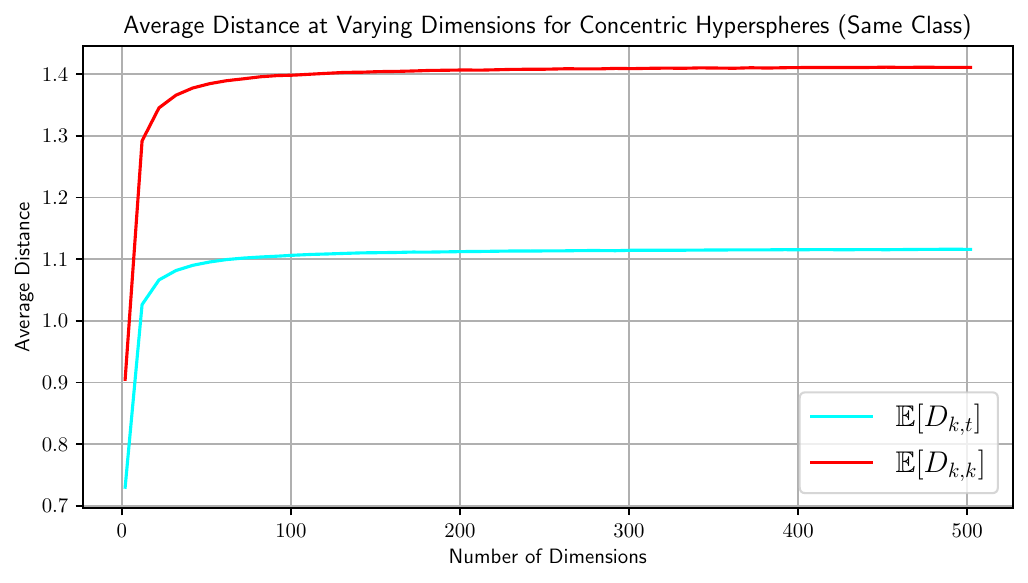}
    }
  \subcaptionbox{Different classes \label{fig:distance_diffclass}}{
        \includegraphics[height=4.7cm]{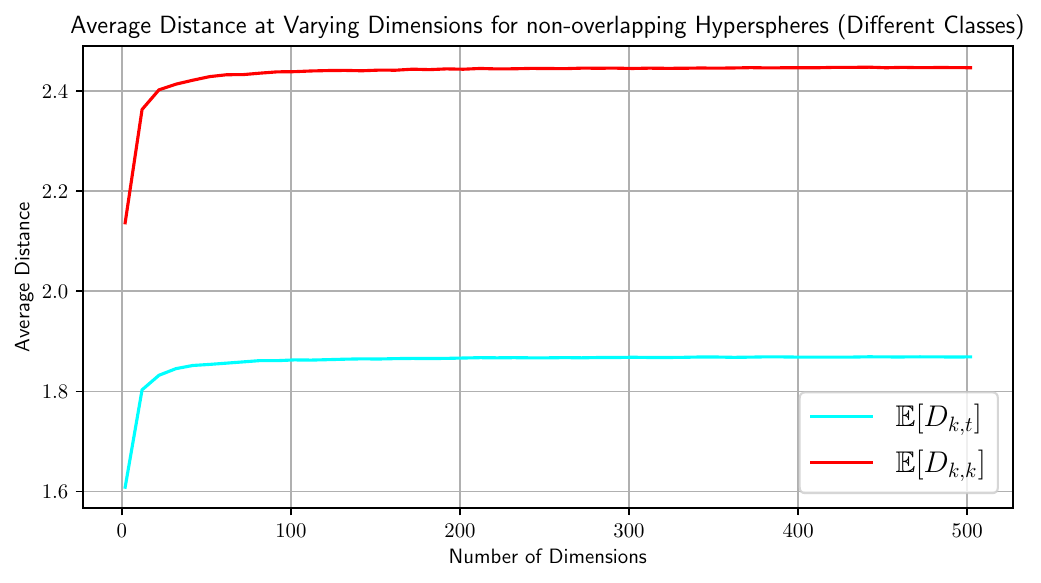}
    }
    \caption{
    Comparison of expected distances between feature points from two learning phases, characterized by indices \(k\) (before learning) and \(t\) (after learning), across different dimensions of the representation space. Both $\mathbb{E}[D_{k,t}]$ and $\mathbb{E}[D_{k,k}]$ are examined. \textit{(a)}: Expected distance in the case of same class, the value of $\mathbb{E}[D_{k,t}]$ remains less than $\mathbb{E}[D_{k,k}]$, satisfying on average the condition of Eq.~\ref{eq:first}. \textit{(b)}: In the case of two different classes, the expected distance, does not satisfy the condition of Eq.~\ref{eq:second}.
    }
    \label{fig:distanceHyperSphere}
\end{figure*}

In contrast, stationary features of models learned through a pre-allocated $d$-Simplex fixed classifier are concentric and do not suffer from class collapse due to pre-allocation. 
Using this result and the three previously established assumptions, the verification of optimality can be achieved. This is done by computing the expected distance according to Eq.~\ref{eq:distance_MC}, particularly within the hyperballs of two models corresponding to a single class.
Expected distance is computed according to Eq.~\ref{eq:distance_MC} by shifting one of the hyperballs and assuming a uniform distribution. Given the symmetry of a hyperball, shifting in any single direction is adequate for the evaluation.
Since no closed form solution of Eq.~\ref{eq:distance_MC} exists Monte Carlo integration is employed. Fig.~\ref{fig:pdisc-shift} illustrates optimality for a corresponding class in two stationary models. It shows that as the amount of shift increases, there is a corresponding increase in the expected distance, a phenomenon observed across various dimensional spaces.

The same evaluation is used to verify the definition of compatibility in Eq.~\ref{eq:first} and Eq.~\ref{eq:second}: 
\begin{equation}
\mathbb{E}[D_{k,t}] \leq \mathbb{E}[D_{k,k}]
\label{eq:expectation_first}
\end{equation}
(in the case of the same class) and if 
\begin{equation}
\mathbb{E}[D_{k,t}] \geq \mathbb{E}[D_{k,k}]
\label{eq:expectation_second}
\end{equation}
(in the case of different classes) hold.
In Fig.~\ref{fig:distance_sameclass} and Fig.~\ref{fig:distance_diffclass}, we show plots of $\mathbb{E}[D_{k,t}]$ and $\mathbb{E}[D_{k,k}]$ with varying feature dimension from $2$ to $500$. 
Without loss of generality, the hyperball radius starts at 
1 and is reduced to 0.5 (further radius reductions follow the same principle and are not shown).
The plots show that as the radius is reduced (i.e., more knowledge is assimilated) in the case of the same class the expected distance $\mathbb{E}[D_{k,t}]$ is always below $\mathbb{E}[D_{k,k}]$ at any feature dimensions (Fig.~\ref{fig:distance_sameclass}). 
Differently, as shown in Fig.~\ref{fig:distance_diffclass}, the expected distance evaluation for the case of different classes results in $\mathbb{E}[D_{k,t}] < \mathbb{E}[D_{k,k}]$ therefore not satisfying Eq.~\ref{eq:second}.
To satisfy Eq.~\ref{eq:second}, the hyperball $\mathcal{B}_t(\mathbf{w}_{y_i})$ from Eq.~\ref{eq:ball_t} should be placed away from the hyperball $\mathcal{B}_k(\mathbf{w}_{y_j})$ of the other class (Eq.~\ref{eq:ball_k}). 
Such repositioning changes the concentric arrangement of the hyperballs of the same class, which negatively affects the optimality.

The optimal approximation to compatibility directly follows from: (1) the fact that hyperballs centered at the vertices of a regular $d$-Simplex, are at their pairwise maximum distance, and (2) the addition of more classes does not alter this distance because their corresponding representation space is pre-allocated and remains unchanged (i.e., stationary). 

\end{proof}

In the proof above, it emerges that the satisfaction of both compatibility constraints of Def.~\ref{def:compatibility-shen} cannot be achieved. In the following corollary, we provide the explicit statement outside the proof above for a clearer and more focused exposition of this result, as it has a general validity beyond the specific assumption of a $d$-Simplex fixed classifier.

\begin{customcor}{1}[Infeasibility]
\label{corollary}
The two compatibility inequalities in Def.~\ref{def:compatibility-shen} cannot be satisfied by the representation learned by a trainable classifier. 
\end{customcor}
\begin{proof}
The proof follows immediately from the arguments presented in the final part of the proof of Theorem~\ref{theo:compatibility}. The discussion therein establishes that in order to satisfy Eq.~\ref{eq:second}, a shift of the hyperball 
$\mathcal{B}_t$ in Eq.~\ref{eq:ball_t} away from the hyperball $\mathcal{B}_k$ in Eq.~\ref{eq:ball_k} is required. 
This results in a departure from the concentric configuration for the case of the same class, thereby negatively affecting the optimality of Eq.~\ref{eq:first}. In the case in which the classifier can be trained, the introduction of additional classes alters the pairwise class distances, and as a result, a departure from the concentric configuration cannot be avoided. As a consequence the inequality constraints of compatibility cannot be satisfied.
\end{proof}

\section{Implementation Details}\label{sec:implementations_details}

In the following section, we provide more detailed information about the experimental settings described in Sec.~\ref{sec:experimental_results}.
We pre-train ResNet18 models on ImageNet32 for $300$ epochs. Pre-training was done using an SGD optimizer with a learning rate of $0.1$, momentum $0.9$, and weight decay $1\cdot10^{-4}$. Models are trained with a mini-batch size of $128$, and the learning rate follows a cosine annealing schedule throughout the training process.
For methods based on the $d$-Simplex fixed classifier \cite{pernici2021regular}, we pre-allocate  $K=1024$ classes (features vectors are then of size $d=1023$) and training is performed according to the cross-entropy loss of Eq.~\ref{eq:loss_ce_simplex}. 
The other methods utilize a trainable classifier, wherein the feature size corresponds to that of the ResNet18 architecture, namely $512$.

Models were fine-tuned on CIFAR100R for $70$ epochs.
Fine-tuning was performed using the SGD optimizer with learning rate of $0.001$, momentum $0.9$, weight decay $10^{-4}$ and with mini-batch size of $128$. The learning rate is decreased according to a linear scheduling with a reduction factor of $0.1$ at epochs $50$ and $65$. 

\section{Ablation Studies}\label{sec:ablation}

In this section, we present ablation studies of $d$-Simplex-HOC using CIFAR100R/10. These studies involved fine-tuning the model for 31 tasks, with two model replacements as is the experiment of Fig.~\ref{fig:cifar30-iamcl2r}.

\subsection{Hyperparamters}
The training of $d$-Simplex-HOC is influenced by the hyperparameters $\lambda$ and $\tau$, as used in Eq.~\ref{eq:total_loss} and Eq.~\ref{eq:CONTRAST2}, respectively. 
Tab.~\ref{tab:ablation_hyperaparams} shows the $AC$ metric for different values of $\lambda$ and $\tau$. The results show that using $\lambda = 0.1$ and $\tau = 10$ yields the highest performance in terms of $AC$. 
A lower value of $\lambda$ suggests a greater emphasis on the contrastive loss relative to the cross-entropy loss, prioritizing the higher-order component over the first-order one offered solely by the cross-entropy.
The value of $\tau$ yielding the highest $AC$ in our study closely aligns with that reported in \cite{chen2020simple}. This similarity suggests a consistent $\tau$ effect across various contexts of representation learning.

\begin{table}[h]
    \small
    \centering
    \setlength{\tabcolsep}{4pt}
    \sisetup{detect-weight=true,detect-all,table-format=-1.3,round-mode=places, round-precision=2}

    \resizebox{\hsize}{!}{
    \begin{tabular}{l|SSSSSSS}

    \arrayrulecolor{black}
    \toprule
    \diagbox{$\lambda$}{$\tau$} & \text{1}  & \text{5}  & \text{8} & \text{10 (${\scalebox{0.6}{$\spadesuit$}}$)} & \text{15} & \text{20} \\
    \midrule
    0.05 & 0.10 & 0.55 & 0.63 & 0.64 & 0.35 & 0.23\\ 
    0.1 (${\scalebox{0.6}{$\spadesuit$}}$) & 0.10 & 0.58 & 0.64 & \textbf{0.65} & 0.36 & 0.23\\ 
    0.25 & 0.06 & 0.30 & 0.43 & 0.42 & 0.34 & 0.21\\ 
    0.5 & 0.09 & 0.23 & 0.19 & 0.20 & 0.18 & 0.21\\ 
    0.75 & 0.17 & 0.19 & 0.16 & 0.13 & 0.12 & 0.10\\ 
    \bottomrule
    \end{tabular}
    }
\caption{Ablation study for $d$-Simplex-HOC in 31 tasks using CIFAR100R/10 with two model replacements of $\lambda$ (Eq.~\ref{eq:total_loss}) and $\tau$ (Eq.~\ref{eq:CONTRAST2}). The evaluation is performed with respect to the $AC$ metric. Values used in our implementation are marked with the ``(${\scalebox{0.8}{$\spadesuit$}}$)'' symbol. }
\label{tab:ablation_hyperaparams}
\end{table}

\subsection{Learning Rate}
Learning a new task without affecting the existing model's representation requires a proper selection of the learning rate.
Tab.~\ref{tab:lr} reports the metrics $AC$ and $AA_{31}$, obtained for different learning rate values $\eta$. A higher $\eta$ enables the model to adapt more quickly to new tasks; however, this results in a noticeable decline in performance with respect to both $AA_{31}$ and $AC$. This decline is primarily due to significant changes in the model's representation before and after the updates. In contrast, a lower learning rate allows the model to transition more gradually from its current state, leading to improved compatibility. This approach, while improving compatibility, results in a slight reduction in the model's ability to assimilate new knowledge from the task. Considering this trade-off, we opted for a learning rate of $0.001$ in our implementation.

\subsection{Training-sets Relative Size} 
We aim to study the impact on performance of the relative size between the dataset used for pre-training the models, namely ImageNet32, and the CIFAR100R dataset used for fine-tuning them.
To this end, we varied the number of images per class in the CIFAR100R dataset.
Tab.~\ref{tab:reduction_cifar} shows the values with 500 (all the images of CIFAR100 are used in CIFAR100R), 300, 200, 100, 50, 10, and 5 images per class.
We observe that compatibility performance ($AC$) decreases as the number of images per class reduces. Conversely, the average accuracy exhibits a gradual decline. This highlights that achieving compatibility is a complex constraint requiring adequate data.

\begin{table}   
\centering
    \begin{subtable}{.45\linewidth}
    \centering    
    \small
    \setlength{\tabcolsep}{4pt}
    \sisetup{detect-weight=true,detect-all,table-format=-1.3,round-mode=places}
    \begin{tabular}{l
                        S[round-precision=2]
                        S[round-precision=2]}
                        
            \arrayrulecolor{black}
            \toprule
            $\eta$ &  $AC$    &  $AA_{31}$ \\
            \midrule
            0.1  &  0.06667 & 58.205 \\
            0.01 & 0.4022 & 68.67  \\ 
            0.005 & 0.5742 & \hspace{-3pt}\textbf{68.94} \\
            0.001 (${\scalebox{0.6}{$\spadesuit$}}$) &  \textbf{0.65} & 67.40 \\
            0.0005 & 0.5677 & 66.31  \\
            0.0001 & 0.3204 & 63.437 \\
            0.00001 & 0.3011 & 63.322 \\
            \bottomrule
        \end{tabular}   
\caption{} \label{tab:lr}
\end{subtable}%
\hspace{10pt}
\begin{subtable}{.45\linewidth}
    \setlength{\tabcolsep}{5pt}
    \sisetup{detect-weight=true,detect-all,table-format=-1.3,round-mode=places}
    \centering
    \small
    \begin{tabular}{l
                    S[round-precision=2]
                    S[round-precision=2]}
                    
        \arrayrulecolor{black}
        \toprule
         \#imgs &  $AC$    &  $AA_{31}$\\
        \midrule
        500  & \textbf{0.69} & \hspace{-3pt}\textbf{67.97} \\
        300 (${\scalebox{0.6}{$\spadesuit$}}$) &  0.6495 & 67.403 \\
        200 & 0.55 & 66.95 \\
        100 & 0.42 & 65.83\\
        50 & 0.3161 & 65.007  \\
        10 & 0.2473 & 62.045  \\
        5 & 0.2215 & 61.769 \\
        \bottomrule
    \end{tabular}
\caption{} \label{tab:reduction_cifar}
\end{subtable}
\vspace{-10pt}
\caption{Ablation for $d$-Simplex-HOC in 31 tasks using CIFAR100R/10 with two model replacements of learning rate $\eta$ (\textit{a}) and of the number of images (\#imgs) per class in CIFAR100R (\textit{b}). Values used in our implementation are marked with the ``(${\scalebox{0.8}{$\spadesuit$}}$)'' symbol. }
    \label{tab:ablation}

\end{table}
\begin{figure}[t]
    \centering
    \includegraphics[width=0.8\linewidth]{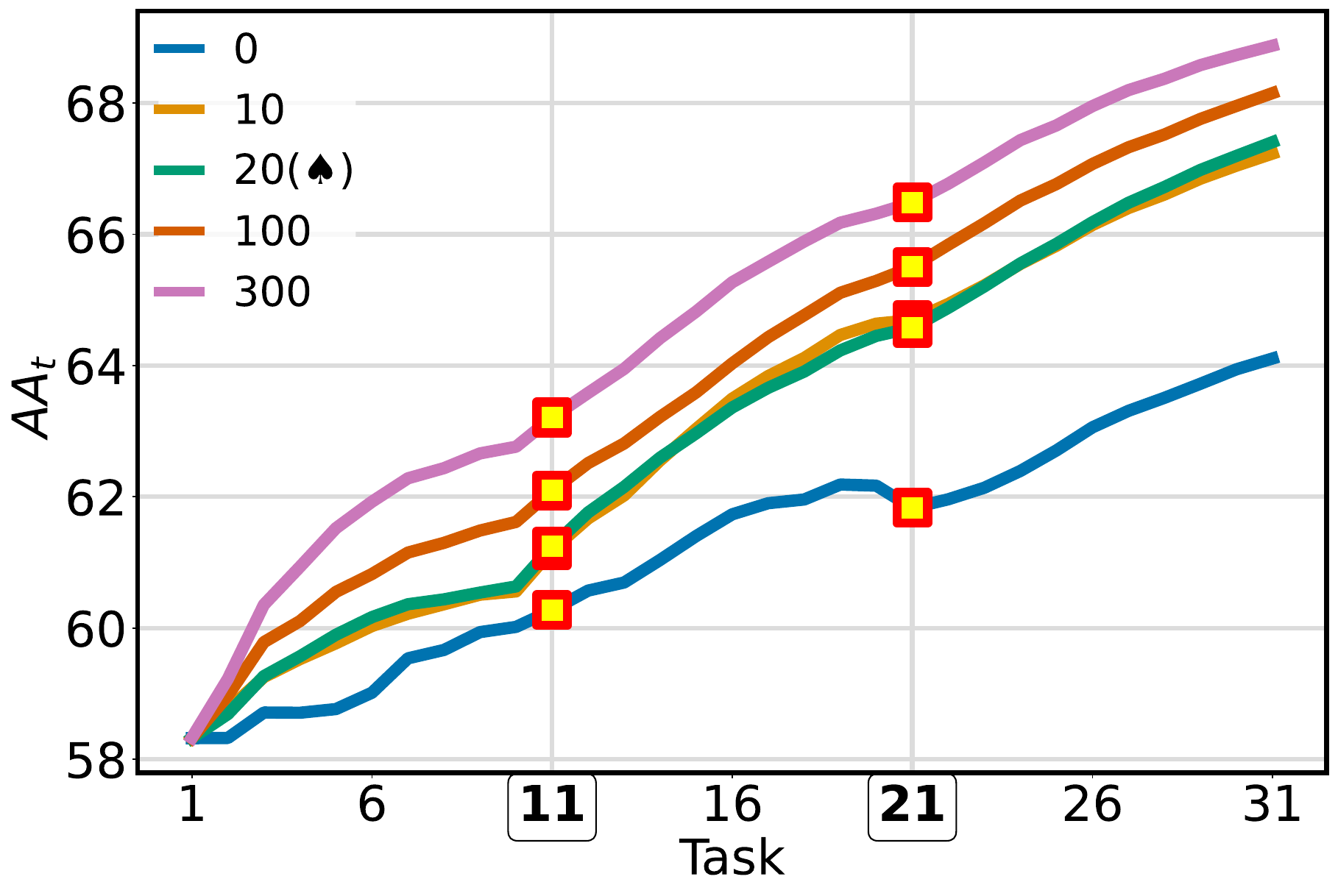}
    \caption{Ablation for $d$-Simplex-HOC in 31 tasks with CIFAR100R/10 with two model replacements of the number of images in the episodic memory (0 is \textit{rehearsal-free}). Values used in our implementation are marked with the ``{(${\scalebox{0.7}{$\spadesuit$}}$)}'' symbol.}\label{fig:abl-memory-size}
\end{figure}

\subsection{Episodic Memory Size}
Fine-tuning is performed using data from the new task along with an episodic memory to mitigate potential forgetting~\cite{icarl}. Consequently, we assess how the number of images per class in the episodic memory impacts the model's performance.
Fig.~\ref{fig:abl-memory-size} shows $AA_{t}$ curves for various numbers of images per class in the episodic memory. These plots illustrate scenarios ranging from the \textit{rehearsal-free} case, where no images are retained, to the case where all images of each class are stored (300 images per class), and include intermediate scenarios as well.
As expected, the more data are used in the memory, the more the accuracy increases. 
Remarkably, in the rehearsal-free case, there is a continuous improvement in accuracy. This case indicates that $d$-Simplex-HOC is capable of leveraging improvements from model replacement, even in the absence of episodic memory. This evidence may be relevant for future search/retrieval  systems which evolve or enhance their performance over time.

\section{\texorpdfstring{$d$}{d}-Simplex fixed classifier PyTorch Code}
\label{sec:simplex_formula}

We provide a GPU-based implementation to generate a $d$-Simplex classifier matrix $\mathbf{W}$ for a given number of pre-allocated classes $K$ that offers faster computation compared to CPU-based implementations~\cite{Pernici_2019_CVPR_Workshops,pernici2021regular,kasarla2022maximum}.

\begin{lstlisting}
def dsimplex_fixed_classifier(K):
    W = torch.zeros((K, K-1))
    W[:-1,:] = torch.eye(K-1)
    W = W.cuda()
    c = torch.sqrt(1 + torch.Tensor([K-1]).cuda())
    W[-1,:] = W[-1,:] + (1 - c) / (K-1)
    W.add_(-torch.mean(W, dim=0))
    W.div_(torch.linalg.norm(W) + 1e-8)
    W.requires_grad = False
    return W
\end{lstlisting} 
\fi

\end{document}